\documentclass[letterpaper]{article}
\usepackage[preprint]{aaai2027}
\usepackage[hyphens]{url}
\usepackage{graphicx}
\usepackage{natbib}
\usepackage{caption}
\usepackage{booktabs}
\usepackage{amsmath}
\usepackage{amssymb}
\usepackage{amsthm}
\usepackage{mathtools}
\usepackage{colortbl}
\newtheorem{proposition}{Proposition}
\newtheorem{theorem}{Theorem}
\newtheorem{lemma}{Lemma}

\title{FUSE: Feature-Wise Unified Specialization with Cross-Column Exchange for Mixed-Type Tabular Flow Matching}

\author{
Suman Cha\textsuperscript{\rm 1},
Seongchan Lee\textsuperscript{\rm 2},
Dohyun Ko\textsuperscript{\rm 1},
Hyunjoong Kim\textsuperscript{\rm 1}\corresponding
}
\affiliations{
\textsuperscript{\rm 1}Department of Statistics and Data Science, Yonsei University\\
\textsuperscript{\rm 2}Department of Mathematical Sciences, KAIST
}

\begin{document}

\renewcommand{\dbltopfraction}{0.95}
\renewcommand{\textfraction}{0.05}

\maketitle

\begin{abstract}
Generating mixed-type tabular data requires jointly modeling diverse feature distributions and their complex cross-column dependencies. Variational flow matching handles distinct endpoints via factorized distributions, yet leaves feature-specific processing and cross-column interactions implicit within a shared backbone. We introduce Feature-wise Unified Specialization with cross-column Exchange (FUSE) to explicitly separate these roles. FUSE applies separate adaptive mixture modules to numerical and categorical features, allowing each feature to combine shared specialized subnetworks, while joint attention preserves information exchange across all columns. We also characterize the excess population risk from restricted conditioning contexts and bound the continuous Wasserstein generation error by endpoint-prediction risk. Comprehensive experiments on eight tabular datasets demonstrate that FUSE achieves strong and consistent performance across distributional fidelity and downstream utility metrics.
\end{abstract}

\section{Introduction}
\label{sec:introduction}

Tabular data is one of the most prevalent data formats in machine learning applications such as healthcare, finance, and public policy \citep{gardner2023benchmarking,holzmuller2024better,little2025synthetic}. The ability to generate synthetic tabular data has become essential for data sharing, mitigating data scarcity, and enabling robust downstream analysis \citep{zhao2021ctabgan,qian2023synthcity,vanbreugel2023testing}. However, generating synthetic samples remains challenging since it requires capturing the complex joint distributions of mixed-type variables.

The challenge extends beyond the distinction between numerical and categorical variables. Columns within the same type can also exhibit different statistical structures \citep{shi2025tabdiff}. Numerical columns may vary in skewness, tail behavior, and multimodality, while categorical columns differ in cardinality and frequency concentration. Generative models therefore need to capture these column-specific distributions while preserving dependencies between the features.

Recent work has adapted flow matching to this setting \citep{jolicoeurmartineau2024forest,guzman2025efvfm}. Flow matching learns a time-dependent velocity field whose induced
ordinary differential equation transports a tractable source distribution
to the data distribution
\citep{liu2023flow,lipman2023flow,albergo2023building}. 
Variational Flow Matching (VFM) expresses the velocity through conditional inference over trajectory endpoints, while Mean-Field Variational Flow Matching (MF-VFM) makes the endpoint inference tractable by factorizing the variational endpoint distribution across variables \citep{eijkelboom2024vfm}. For mixed-type data, Exponential Family Variational Flow Matching (EF-VFM) extends this approach by assigning appropriate exponential family distributions to match the statistical properties of each variable type \citep{guzman2025efvfm}. This factorization decomposes the training objective across columns without restricting the information available to their predictors. Each endpoint factor may condition on all coordinates of the intermediate state.

While endpoint factorization aligns the objective with varying data types, the architectural challenge of processing diverse variables remains unresolved. TabbyFlow implements EF-VFM by processing column embeddings through a shared backbone \citep{guzman2025efvfm}. Although this design supports parameter sharing and encodes feature identity through column embeddings, it provides no explicit mechanism for feature-dependent transformations.

Motivated by this gap, we propose FUSE, a
Feature-wise Unified Specialization with cross-column
Exchange architecture for mixed-type tabular flow matching.
FUSE applies feature-dependent component recombination with joint
self-attention, enabling specialized computation while preserving
cross-type information exchange. Within each variable type, adaptive mixture processing forms feature-dependent combinations of shared specialized subnetworks using a differentiable aggregate-transform-recombine operator \citep{puigcerver2024soft}. Joint attention subsequently connects all column representations, preserving information exchange between numerical and categorical variables. FUSE therefore separates feature-specific processing from the conditioning information available to each endpoint predictor.

The contributions of this work are summarized as follows.

\begin{enumerate}
    \item We introduce a mixed-type flow-matching architecture that combines type-specific adaptive mixture processing with joint attention across all columns.

    \item We quantify the exact excess population risk induced by restricting the conditioning context. Under regularity conditions, we further bound the continuous-space Wasserstein generation error in terms of endpoint-prediction risk.

    \item Comprehensive experiments on eight tabular datasets show that FUSE achieves competitive or superior generation quality across diverse metrics.
\end{enumerate}
\section{Flow Matching for Mixed-Type Data}
\label{sec_mean_field_vfm}

We review the flow matching formulation for mixed-type data. Flow matching defines a continuous transport through a time-dependent velocity field, while variational flow matching represents this velocity through conditional endpoint inference. Mean-field and exponential-family extensions yield tractable objectives for numerical and categorical endpoints.

\subsection{Flow Matching}

Let $X_1^{\mathrm{num}}\in\mathbb R^{d_{\mathrm{num}}}$ denote the numerical endpoint and let $C_1^{(k)}\in\{1,\ldots,K_k\}$ denote categorical endpoint $k$ for $k=1,\ldots,d_{\mathrm{cat}}$. Using the one-hot map $e_k:\{1,\ldots,K_k\}\to\{0,1\}^{K_k}$, we write
\begin{equation*}
X_1
=
\bigl(
X_1^{\mathrm{num}},
e_1(C_1^{(1)}),\ldots,
e_{d_{\mathrm{cat}}}(C_1^{(d_{\mathrm{cat}})})
\bigr)
\in\mathbb R^D,
\end{equation*}
where $D=d_{\mathrm{num}}+\sum_{k=1}^{d_{\mathrm{cat}}}K_{k}.$ Let $X_0\sim p_0=\mathcal N(0,I_D)$ be independent of $X_1$. The linear conditional path is
\begin{equation}
    X_t=(1-t)X_0+tX_1,
    \qquad t\in[0,1).
\label{eq:linear_conditional_path}
\end{equation}

Flow matching estimates the marginal velocity of this path and generates samples by integrating the learned ordinary differential equation \citep{lipman2023flow,albergo2023stochastic}. For a fixed endpoint $x_1$, the conditional velocity is
$
u_t(x\mid x_1)
=
(x_1-x)/(1-t).
$
The marginal velocity is its conditional expectation given $X_t=x$
\begin{equation*}
\begin{aligned}
u_t(x)
&=
\mathbb E
\left[
u_t(x\mid X_1)
\mid X_t=x
\right] \\
&=
\frac{m_t(x)-x}{1-t},
\qquad
m_t(x)
=
\mathbb E[X_1\mid X_t=x].
\end{aligned}
\end{equation*}
Thus, the endpoint posterior determines the marginal velocity only through its mean $m_{t}(x)$.

\subsection{Variational Endpoint Inference}

Let $p_{1\mid t}(x_1\mid x)$ denote the endpoint posterior at time $t$. Variational Flow Matching \citep{eijkelboom2024vfm} approximates this posterior with $q_{\theta,t}(x_1\mid x)$ by minimizing
\begin{equation*}
\mathcal L_{\mathrm{VFM}}(\theta)
=
-\mathbb E
\left[
\log q_{\theta,t}(X_1\mid X_t)
\right],
\end{equation*}
where the expectation is taken over $t\sim\operatorname{Unif}(0,1)$, $X_0\sim p_0$, and $X_1\sim p_1$. The expectation of the variational endpoint distribution is
\begin{equation*}
\widehat m_{\theta,t}(x)
=
\mathbb E_{q_{\theta,t}(\cdot\mid x)}[X_1],
\end{equation*}
which induces the vector field
\begin{equation*}
v_{\theta,t}(x)
=
\frac{\widehat m_{\theta,t}(x)-x}{1-t}.
\end{equation*}

Mean-Field Variational Flow Matching imposes a product form on the variational endpoint distribution. Exponential Family Variational Flow Matching \citep{guzman2025efvfm} extends this construction to mixed-type data by assigning an appropriate exponential-family distribution to each endpoint column. Combining these two ingredients, we use Gaussian factors for numerical endpoints and categorical factors for categorical endpoints
\begin{equation}
\begin{aligned}
q_{\theta,t}(x_1\mid x)
&=
\prod_{j=1}^{d_{\mathrm{num}}}
\mathcal N\!\left(
x_{1,j}^{\mathrm{num}}
\mid
\mu_{\theta,j,t}(x),
\nu_t
\right) \\
&\quad\times
\prod_{k=1}^{d_{\mathrm{cat}}}
\operatorname{Cat}\big(
c_1^{(k)}
\mid
\pi_{\theta,k,t}(x)
\big).
\end{aligned}
\label{eq:endpoint_family}
\end{equation}
where $\nu_t>0$ is a predetermined variance. Under this distribution, the conditional expectation of each numerical endpoint is $\mu_{\theta,j,t}(x)$. For a one-hot categorical block, its conditional expectation is the corresponding probability vector $\pi_{\theta,k,t}(x)$. Up to additive terms independent of $\theta$, the negative log-likelihood is then
\begin{equation}
\begin{aligned}
\mathcal L_{\mathrm{end}}(\theta)
&=
\mathbb E\Bigg[
\frac{1}{2\nu_t}
\left\|
X_1^{\mathrm{num}}
-
\mu_{\theta,t}(X_t)
\right\|_2^2
\\[-1mm]
&\qquad\qquad
-
\sum_{k=1}^{d_{\mathrm{cat}}}
\log
\pi_{\theta,k,t}(X_t)_{C_1^{(k)}}
\Bigg].
\end{aligned}
\label{eq:endpoint_training_objective}
\end{equation}
The mean-field assumption factorizes $q_{\theta,t}(\cdot\mid x)$ across endpoint columns, but places no restriction on the conditioning argument $x$. Consequently, each column-wise parameter, including $\mu_{\theta,j,t}(x)$ and $\pi_{\theta,k,t}(x)$, may depend on all coordinates of $x$. The variational objective does not specify how computation should be shared effectively among these column-wise parameter functions.
\section{Proposed Architecture}
\label{sec:proposed_architecture}

\begin{figure*}[t!]
    \centering
    \includegraphics[width=1.0\textwidth]{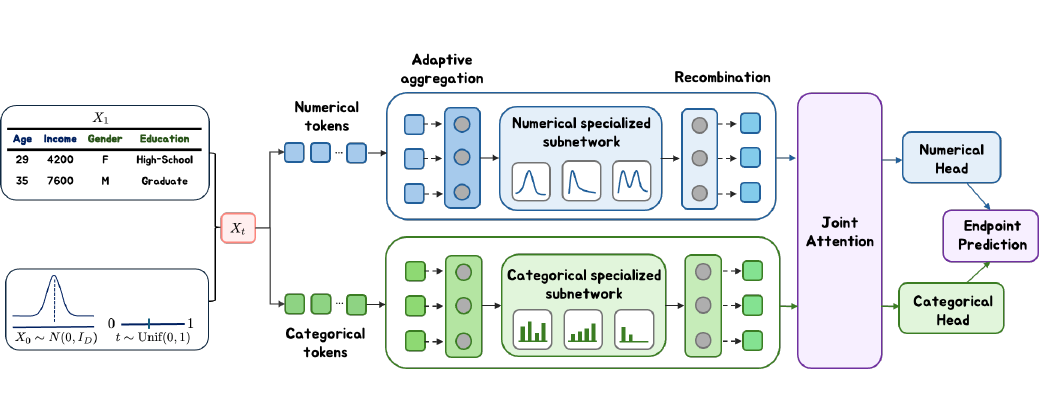}
    \vspace{-4em}
    \caption{Overview of the FUSE architecture. Type-specific adaptive mixture processing transforms numerical and categorical tokens into specialized representations, which then interact via global attention to capture multivariate feature dependencies.}
    \label{fig:model_architecture}
\end{figure*}

While the VFM-based approaches in Section~\ref{sec_mean_field_vfm} specify the factorized endpoint distributions, their effective neural parameterization remains open. For mixed-type tabular data, a fully shared backbone lacks feature-specific computation, whereas independent networks sacrifice beneficial parameter sharing. To resolve this tradeoff, we propose FUSE, an architecture interleaving type-specific adaptive mixture processing with joint attention, enabling feature specialization while maintaining cross-column dependencies.

\subsection{Adaptive Mixture Processing}

Let $X_t = (X_t^{\mathrm{num}}, X_t^{\mathrm{cat}})$ denote the intermediate state at time $t$. We project each feature into an initial token $h_{r,j,t}^{(0)} \in \mathbb{R}^{d_h}$ for type $r \in \{\mathrm{num}, \mathrm{cat}\}$ via feature-specific linear layers, learnable embeddings, and a time embedding $\psi(t)$. This yields embedding matrices $H_{\mathrm{num},t}^{(0)} \in \mathbb{R}^{d_{\mathrm{num}}\times d_h}$ and $H_{\mathrm{cat},t}^{(0)} \in \mathbb{R}^{d_{\mathrm{cat}}\times d_h}$.

At layer $\ell$, the adaptive mixture module for type $r$ employs $M_r$ parallel subnetworks to process $S_r = M_r P_r$ latent components. Using the projection $\Phi_r^{(\ell)}\in\mathbb{R}^{d_h\times S_r}$ and tokens $\overline H_{r,t}^{(\ell)} = \operatorname{LN} (H_{r,t}^{(\ell)})$, we compute the alignment scores as 
\begin{equation}
R_{r,t}^{(\ell)} = \frac{\overline H_{r,t}^{(\ell)} \Phi_r^{(\ell)}}{\sqrt{d_h}} \in \mathbb{R}^{n_r\times S_r}.
\label{eq:mixture_scores}
\end{equation}
The score matrix is then normalized along two axes. The aggregation weights
$
D_{r,t}^{(\ell)}
=
\operatorname{softmax}
(R_{r,t}^{(\ell)})
$
are normalized across feature tokens for each component and determine how the feature representations are aggregated into latent components. The recombination weights
$
C_{r,t}^{(\ell)}
=
\operatorname{softmax}
(R_{r,t}^{(\ell)})
$
are normalized across components for each feature and determine how the processed components are combined to update each feature representation.

These two normalizations control the aggregation and recombination stages of the module. The aggregation weights first form the latent components 
$$U_{r,t}^{(\ell)} = (D_r^{(\ell)})^\top \;\overline H_{r,t}^{(\ell)}.$$ 
The $S_r$ components are partitioned among the $M_r$ subnetworks, with $P_r$ components assigned to each subnetwork. Let $\kappa_r(s)=\lceil s/P_r\rceil$ denote the subnetwork assigned to component $s$. Each component is then transformed as
\[
V_{r,s,t}^{(\ell)}
=
g_{r,\kappa_r(s)}^{(\ell)}
\left(U_{r,s,t}^{(\ell)}\right),
\qquad s=1,\ldots,S_r,
\]
and the resulting representations are stacked into $V_{r,t}^{(\ell)}\in\mathbb{R}^{S_r\times d_h}$. The recombination weights distribute these transformed components back to the feature tokens, producing the residual update
\begin{equation}
\widetilde H_{r,t}^{(\ell)}
=
H_{r,t}^{(\ell)}
+
C_{r,t}^{(\ell)}V_{r,t}^{(\ell)}.
\label{eq:mixture_update}
\end{equation}
Since both weights depend on the feature representations, each feature can receive a distinct combination of subnetwork transformations. The subnetworks are also shared across all features of type $r$, allowing type-specific specialization without introducing a separate subnetwork for every feature.

\subsection{Joint Attention}

The adaptive mixture module processes numerical and categorical tokens through separate sets of subnetworks. This separation supports type-specific transformations, but does not allow by itself information exchange across feature types. To enable joint contextualization, we concatenate the type-specific outputs $\widetilde H_t^{(\ell)} = \operatorname{Concat} ( \widetilde H_{\mathrm{num},t}^{(\ell)}, \widetilde H_{\mathrm{cat},t}^{(\ell)} )$ and apply multi-head self-attention \citep{vaswani2017attention}:
\begin{equation}
H_t^{(\ell+1)} = \widetilde H_t^{(\ell)} + \operatorname{MHA}^{(\ell)} \big( \operatorname{LN} \big( \widetilde H_t^{(\ell)} \big) \big).
\label{eq:global_attention}
\end{equation}
This operation allows each feature token to incorporate information from all numerical and categorical features. The resulting sequence is then partitioned into $H_{\mathrm{num},t}^{(\ell+1)}$ and $H_{\mathrm{cat},t}^{(\ell+1)}$ before entering the next layer. Each layer therefore combines type-specific adaptive processing with joint interaction.

\subsection{Endpoint Prediction}

After $L$ blocks, feature-specific output heads map the final numerical and categorical tokens to the parameters of their corresponding endpoint factors 
\begin{equation*}
\begin{aligned}
\mu_{\theta,j,t}(X_t)
&=
\left(w_{\mu,j}\right)^\top
h_{\mathrm{num},j,t}^{(L)}
+
b_{\mu,j}, \\
\pi_{\theta,k,t}(X_t)
&=
\operatorname{softmax}\left(
W_{\pi,k}h_{\mathrm{cat},k,t}^{(L)}
+
b_{\pi,k}
\right).
\end{aligned}
\label{eq:endpoint_heads}
\end{equation*}
These outputs parameterize the factorized endpoint distribution in Equation~\eqref{eq:endpoint_family} and are optimized using the endpoint objective in Equation~\eqref{eq:endpoint_training_objective}. 

Each head operates on a feature-specific representation containing both adaptive transformations and joint context. Concatenating the predicted numerical means and categorical probabilities yields $\widehat{m}_{\theta,t}(x)$, which defines the vector field
$
v_{\theta,t}(x)
=
(\widehat m_{\theta,t}(x)-x)/(1-t).
$

\section{Theoretical Analysis}
\label{sec:theory_analysis}

The mean-field assumption determines only how the endpoint distribution factorizes, leaving unspecified both the information available to each endpoint function and the representations shared among them. We first use an illustrative construction to examine the approximation error induced by restricting distinct endpoint functions to a single shared representation. We then show that endpoint excess risk controls the Wasserstein error of the generated distribution. All proofs and derivations are provided in the Appendix.

\paragraph{Conditioning Penalty.}
Fix $T\in(0,1)$. Let $t$ be independent of $(X_0,X_1)$ with density $\rho$ on $[0,T]$, and let $\mathcal F=\sigma(t,X_t)$ denote the full conditioning $\sigma$-algebra, where $X_t$ follows the conditional path in Equation~\eqref{eq:linear_conditional_path}. 
Let $\mathcal G$ be a restricted conditioning $\sigma$-algebra satisfying
$
\sigma(t)\subseteq\mathcal G\subseteq\mathcal F.
$
For any $\sigma$-algebra $\mathcal A$ satisfying $\sigma(t)\subseteq\mathcal A\subseteq\mathcal F$, define the numerical posterior mean $m_{\mathcal A} = \mathbb E[X_1^{\mathrm{num}}\mid\mathcal A]$ and the categorical posterior probability vector $p_{k,\mathcal A}$ for $C_1^{(k)}$. Let $\mathcal R_T^*(\mathcal A)$ denote the minimum endpoint risk over all $\mathcal A$-measurable numerical predictors and categorical probability vectors.

\begin{proposition}
\label{prop:conditioning_penalty}
Suppose $\mathbb E\|X_1^{\mathrm{num}}\|_2^2<\infty$ with $0<\nu_t<\infty$ almost surely. Then
\begin{align*}
\mathcal R_T^*(\mathcal G)-\mathcal R_T^*(\mathcal F)
={}&
\mathbb E\!\left[
\frac{\|m_{\mathcal F}-m_{\mathcal G}\|_2^2}{2\nu_t}
\right] \\
&+
\sum_{k=1}^{d_{\mathrm{cat}}}
\mathbb E\!\left[
\operatorname{KL}
\bigl(p_{k,\mathcal F}\,\|\,p_{k,\mathcal G}\bigr)
\right]
\ge 0.
\end{align*}
Equality holds if and only if $m_{\mathcal F}=m_{\mathcal G}$ and $p_{k,\mathcal F}=p_{k,\mathcal G}$ for every $k$, almost surely.
\end{proposition}

The decomposition vanishes when the restricted context preserves both the numerical posterior mean and every categorical posterior probability. A penalty arises only when the omitted features contain residual endpoint information after conditioning on the retained features.

\paragraph{Architectural Bottlenecks.}
Let $S$ be a Rademacher variable that satisfies $\mathbb P(S=1)=\mathbb P(S=-1)=1/2$. Let $\xi\sim\mathcal N(0,\sigma^2)$ and $W,V\stackrel{\mathrm{i.i.d.}}{\sim}\mathcal N(0,1)$ be mutually independent and independent of $S$. Define $U=\delta S+\xi$ for $\delta\neq0$, and suppose $\nu_t\equiv\nu>0$.

\smallskip
\noindent\textit{$(i)$ Conditioning bottleneck. }
Suppose a numerical endpoint has conditional mean $m^{(1)}(U,S)=\alpha U+cS$ for $c\neq 0$. If its predictor is conditioned on $U$ but not on $S$, Proposition~\ref{prop:conditioning_penalty} implies the excess numerical endpoint risk
\begin{align*}
\frac{c^2}{2\nu}\,
\mathbb E\!\left[\operatorname{Var}(S\mid U)\right]
&=
\frac{c^2}{2\nu}\,
\mathbb E\bigg[
\operatorname{sech}^2
\bigg(\frac{\delta U}{\sigma^2}\bigg)
\bigg]
>0.
\end{align*}
The equality follows from $\mathbb E[S\mid U]=\tanh(\delta U/\sigma^2)$. The positive penalty shows that an informative numerical proxy for $S$ need not be sufficient for the endpoint mean when $\sigma^2>0$. Joint attention allows the categorical representation to enter the numerical feature update and therefore avoids imposing the restricted conditioning used in this example. 

\smallskip
\noindent\textit{$(ii)$ Representation bottleneck. }
Now consider two additional endpoint means 
\[
m^{(2)}(W,V)=aW,
\qquad
m^{(3)}(W,V)=bV,
\]
where $a,b\neq 0$. Suppose both endpoint functions are restricted to a
single shared scalar representation
\begin{align*}
\mathcal H_{\mathrm{shared}}
=
\left\{
(\theta_2 h,\theta_3 h):
\|h\|_{L^2}=1,\;
\theta_2,\theta_3\in\mathbb R
\right\}.
\end{align*}
The orthogonality of $W$ with $V$ implies
\begin{align*}
&\inf_{(f_2,f_3)\in\mathcal H_{\mathrm{shared}}}
\mathbb E\Big[
\{m^{(2)}-f_2\}^2+\{m^{(3)}-f_3\}^2
\Big] \\
&\qquad=
\min\{a^2,b^2\}
>0.
\end{align*}
By contrast, the shared component dictionary represents both endpoint means exactly when the features utilize divergent recombination weights, demonstrated in Appendix Section~A. This construction shows that feature-dependent recombination can remove the approximation error induced by the specified single-representation restriction while retaining a shared set of transformations.

The two examples concern distinct architectural restrictions. Joint attention allows cross-type information, whereas adaptive mixture processing allows different features to form distinct combinations of shared transformations. Their relevance depends respectively on residual cross-type endpoint information and the adequacy of a shared representation.

\paragraph{Distributional Error Bound.}
Let $f^*$ be the unrestricted Bayes predictor and define the excess risk $\mathcal E_T(f)=\mathcal R_T(f)-\mathcal R_T(f^*)$. The following theorem bounds the Wasserstein error of the generated distribution in terms of $\mathcal E_T(f)$.

\begin{theorem}
\label{thm:wasserstein_control}
Fix $T \in (0,1)$ and let $t$ have density $\rho$ on $[0,T]$. Assume $0 <\underline{\rho} < \rho(t)$, $0<\underline{\nu}\leq\nu_t\leq\overline{\nu}<\infty$, and $\mathbb E\|X_1^{\mathrm{num}}\|_2^2<\infty$. Suppose $v_{\theta,t}$ is jointly measurable in $(t,x)$, uniformly $L$-Lipschitz in $x$ over $t$, and satisfies
\[
\int_0^T\|v_{\theta,t}(0)\|_2\,dt<\infty.
\]
Then its time-$T$ distribution $\widehat p_T^\theta$ satisfies
\begin{align*}
W_2(\widehat p_T^\theta,p_1)
&\leq
\frac{\sqrt{K_{\rho,\nu}\Gamma_L(T)}}{1-T}
\sqrt{\mathcal E_T(\theta)} \\
&\quad+
(1-T)
\bigl(\mathbb E\|X_0-X_1\|_2^2\bigr)^{1/2},
\end{align*}
where
$K_{\rho,\nu}=2\max\{\overline{\nu},1\}/\underline{\rho}$ and $\Gamma_L(T)=(e^{2LT}-1)/(2L)$, with $\Gamma_0(T)=T$.
\end{theorem}

The first term converts endpoint excess risk into a Wasserstein error bound, while the second captures truncation at $T<1$. Hence, larger $T$ reduces truncation error but increases the risk-dependent coefficient.

\begin{table*}[t]
\centering
\small
\setlength{\tabcolsep}{1.5mm}
\begin{tabular}{lrrrrrrrrr}
\toprule
\multicolumn{10}{l}{\textbf{Shape ($\uparrow$)}} \\
Method & Adult & Default & Beijing & Shoppers & Magic & News & Diabetes & Fault & Avg. rank ($\downarrow$) \\
\cmidrule(lr){1-10}
CTGAN & 0.821(.001) & 0.865(.000) & 0.802(.001) & 0.777(.001) & 0.903(.001) & 0.808(.000) & 0.695(.005) & 0.732(.002) & 6.88 \\
TVAE & 0.852(.001) & 0.897(.001) & 0.771(.001) & 0.759(.001) & 0.921(.002) & 0.830(.001) & 0.812(.006) & 0.857(.002) & 6.12 \\
CoDi & 0.777(.001) & 0.784(.001) & 0.801(.001) & 0.702(.001) & 0.903(.003) & 0.687(.001) & 0.800(.014) & 0.777(.040) & 7.25 \\
TabDDPM & 0.990(.001) & 0.987(.001) & 0.977(.001) & 0.971(.001) & 0.992(.001) & 0.187(.000) & 0.930(.088) & 0.302(.208) & 5.12 \\
TabSyn & 0.979(.001) & 0.954(.001) & 0.967(.001) & 0.983(.001) & 0.989(.001) & \underline{0.986}(.000) & 0.947(.006) & 0.959(.003) & 4.12 \\
TabbyFlow & \textbf{0.993}(.001) & \underline{0.989}(.001) & \underline{0.991}(.001) & 0.986(.001) & 0.991(.001) & 0.975(.000) & \underline{0.952}(.004) & \underline{0.970}(.002) & \underline{2.50} \\
TabDiff & 0.992(.000) & 0.987(.001) & 0.989(.000) & \underline{0.987}(.002) & \textbf{0.992}(.001) & 0.975(.000) & 0.952(.005) & 0.938(.004) & 2.75 \\
\rowcolor[gray]{0.94}
FUSE & \underline{0.993}(.001) & \textbf{0.992}(.001) & \textbf{0.992}(.001) & \textbf{0.989}(.001) & \underline{0.992}(.001) & \textbf{0.988}(.001) & \textbf{0.957}(.004) & \textbf{0.972}(.002) & \textbf{1.25} \\
\midrule
\multicolumn{10}{l}{\textbf{Trend ($\uparrow$)}} \\
Method & Adult & Default & Beijing & Shoppers & Magic & News & Diabetes & Fault & Avg. rank ($\downarrow$) \\
\cmidrule(lr){1-10}
CTGAN & 0.832(.016) & 0.839(.001) & 0.889(.003) & 0.845(.002) & 0.888(.005) & 0.884(.001) & 0.650(.053) & 0.784(.004) & 7.25 \\
TVAE & 0.857(.012) & 0.931(.001) & 0.841(.003) & 0.900(.002) & 0.930(.010) & 0.932(.002) & 0.695(.027) & 0.892(.004) & 6.25 \\
CoDi & 0.810(.001) & 0.640(.088) & 0.916(.006) & 0.914(.001) & 0.941(.004) & 0.863(.002) & 0.822(.039) & 0.920(.009) & 6.38 \\
TabDDPM & 0.981(.001) & 0.948(.004) & \textbf{0.977}(.001) & 0.956(.004) & 0.987(.004) & 0.602(.008) & 0.897(.147) & 0.420(.213) & 5.00 \\
TabSyn & 0.954(.001) & 0.920(.014) & 0.955(.003) & 0.982(.002) & \underline{0.992}(.002) & \underline{0.983}(.002) & 0.928(.010) & 0.969(.004) & 3.75 \\
TabbyFlow & \underline{0.987}(.001) & \underline{0.988}(.001) & \underline{0.972}(.004) & \underline{0.983}(.002) & \textbf{0.993}(.001) & \textbf{0.984}(.002) & \underline{0.936}(.009) & \underline{0.978}(.002) & \textbf{1.75} \\
TabDiff & 0.984(.001) & 0.986(.001) & 0.967(.003) & 0.976(.002) & 0.991(.001) & 0.941(.002) & 0.929(.008) & 0.818(.011) & 3.62 \\
\rowcolor[gray]{0.94}
FUSE & \textbf{0.988}(.001) & \textbf{0.989}(.001) & 0.966(.006) & \textbf{0.988}(.002) & 0.989(.005) & 0.983(.003) & \textbf{0.966}(.005) & \textbf{0.985}(.001) & \underline{2.00} \\
\midrule
\multicolumn{10}{l}{\textbf{C2ST ($\uparrow$)}} \\
Method & Adult & Default & Beijing & Shoppers & Magic & News & Diabetes & Fault & Avg. rank ($\downarrow$) \\
\cmidrule(lr){1-10}
CTGAN & 0.568(.004) & 0.723(.003) & 0.803(.003) & 0.644(.009) & 0.578(.006) & 0.588(.009) & 0.143(.007) & 0.175(.010) & 6.50 \\
TVAE & 0.553(.003) & 0.561(.003) & 0.659(.011) & 0.301(.011) & 0.729(.004) & 0.466(.005) & 0.439(.018) & 0.642(.015) & 6.62 \\
CoDi & 0.411(.003) & 0.268(.002) & 0.632(.002) & 0.410(.003) & 0.737(.004) & 0.081(.001) & 0.559(.031) & 0.471(.107) & 7.00 \\
TabDDPM & 0.970(.004) & 0.978(.005) & 0.970(.003) & 0.886(.006) & 0.993(.006) & 0.000(.000) & 0.952(.179) & 0.008(.009) & 5.12 \\
TabSyn & 0.912(.006) & 0.847(.004) & 0.883(.003) & 0.962(.009) & 0.991(.007) & \underline{0.970}(.004) & 0.980(.023) & 0.851(.058) & 4.12 \\
TabbyFlow & \textbf{0.992}(.005) & \underline{0.987}(.005) & \underline{0.991}(.005) & 0.971(.009) & \textbf{0.998}(.004) & 0.922(.004) & \textbf{0.996}(.008) & \textbf{0.996}(.007) & \underline{1.75} \\
TabDiff & 0.981(.005) & 0.964(.003) & 0.972(.004) & \textbf{0.982}(.008) & 0.995(.004) & 0.895(.004) & 0.983(.022) & 0.635(.027) & 3.25 \\
\rowcolor[gray]{0.94}
FUSE & \underline{0.991}(.005) & \textbf{0.989}(.010) & \textbf{0.996}(.004) & \underline{0.979}(.011) & \underline{0.997}(.004) & \textbf{0.973}(.008) & \underline{0.993}(.010) & \underline{0.994}(.007) & \textbf{1.62} \\
\midrule
\multicolumn{10}{l}{\textbf{MLE}} \\
Method & Adult & Default & Beijing & Shoppers & Magic & News & Diabetes & Fault & Avg. rank ($\downarrow$) \\
& AUC ($\uparrow$) & AUC ($\uparrow$) & RMSE ($\downarrow$) & AUC ($\uparrow$) & AUC ($\uparrow$) & RMSE ($\downarrow$) & AUC ($\uparrow$) & AUC ($\uparrow$) & \\
\cmidrule(lr){1-10}
CTGAN & 0.889(.002) & 0.744(.007) & 0.865(.017) & 0.851(.009) & 0.871(.003) & 0.894(.013) & 0.489(.105) & 0.772(.034) & 6.88 \\
TVAE & 0.878(.002) & 0.747(.005) & 0.873(.019) & 0.872(.006) & 0.890(.003) & 0.995(.015) & 0.808(.026) & 0.880(.008) & 6.62 \\
CoDi & 0.744(.018) & 0.483(.024) & 0.770(.029) & 0.819(.015) & 0.926(.002) & 1.771(.161) & \underline{0.820}(.024) & 0.948(.010) & 6.00 \\
TabDDPM & 0.907(.001) & \textbf{0.770}(.002) & 0.608(.018) & 0.916(.005) & 0.931(.003) & 2.510(.890) & 0.816(.070) & 0.511(.030) & 4.50 \\
TabSyn & 0.908(.001) & 0.757(.009) & 0.703(.020) & 0.909(.005) & \textbf{0.932}(.003) & 0.883(.019) & \textbf{0.823}(.034) & 0.940(.012) & 3.50 \\
TabbyFlow & \underline{0.910}(.003) & 0.764(.007) & \textbf{0.541}(.013) & 0.917(.006) & \underline{0.931}(.003) & \underline{0.872}(.009) & 0.810(.030) & \underline{0.960}(.007) & \underline{2.75} \\
TabDiff & \textbf{0.912}(.002) & 0.768(.007) & \underline{0.573}(.013) & \underline{0.919}(.004) & 0.930(.002) & 0.888(.026) & 0.818(.037) & 0.929(.006) & 3.12 \\
\rowcolor[gray]{0.94}
FUSE & 0.909(.001) & \underline{0.769}(.004) & 0.610(.004) & \textbf{0.921}(.003) & 0.930(.003) & \textbf{0.841}(.003) & 0.810(.023) & \textbf{0.962}(.006) & \textbf{2.62} \\
\bottomrule
\end{tabular}
\caption{Performance comparison across eight datasets. Results are reported as means with standard deviations in parentheses. Bold and underlined values indicate the best and second-best results, respectively, and the proposed method is shaded.}
\label{tab:baseline-detailed}
\end{table*}

\section{Experiments}
\label{sec:experiments}

We evaluate the proposed architecture by comparing it against various baseline models across multiple datasets, assessing both distributional fidelity and downstream machine learning utility. We complement this evaluation with qualitative analyses that visually assess the preservation of marginal distributions and pairwise feature correlations. Finally, we conduct ablation studies to examine the respective contributions of joint attention and adaptive mixture processing.

\subsection{Experimental Setup}
\label{sec:experimental_setup}

\paragraph{Datasets.}
We evaluate FUSE on eight tabular datasets: Adult, Default, Beijing, Shoppers, Magic, News, Diabetes, and Fault. Each dataset contains both numerical and categorical features and is associated with either binary classification, multiclass classification, or regression. Dataset sizes range from 691 to 37,581 samples. Detailed dataset profiles are provided in the Appendix.

\paragraph{Baselines.}
We compare FUSE against seven competitive synthetic tabular data generation methods spanning four model families: 1) GAN-based method: CTGAN \citep{xu2019ctgan}; 2) VAE-based method: TVAE \citep{xu2019ctgan}; 3) Diffusion-based methods: CoDi \citep{lee2023codi}, TabDDPM \citep{kotelnikov2023tabddpm}, TabSyn \citep{zhang2024tabsyn}, and TabDiff \citep{shi2025tabdiff}; and 4) Flow-based methods: TabbyFlow \citep{guzman2025efvfm}.

\paragraph{Metrics.}
We evaluate generative fidelity and downstream utility using diverse metrics. Shape quantifies marginal distributional similarity using Kolmogorov-Smirnov statistics for numerical features and total variation distances for categorical features. Trend measures the preservation of pairwise feature dependencies by comparing the correlation structures between real and synthetic data. We also report a Classifier Two-sample Test (C2ST) score, for which a higher value indicates that a discriminator is less able to distinguish real from synthetic samples. To assess downstream utility, Machine Learning Efficiency (MLE) trains a predictive model on synthetic data and evaluates its performance on held-out real data. This section reports Shape, Trend, C2ST, and MLE, while results for $\alpha$-Precision and $\beta$-Recall are provided in the Appendix.

\paragraph{Implementation Details.}
Each main result is averaged over 20 random initialization seeds. For C2ST, we use an XGBoost classifier to discriminate real from synthetic samples. Complete implementation details, including preprocessing procedures and hyperparameter settings, are provided in the Appendix.

\begin{figure}[t]
    \centering
    \includegraphics[width=1.0\linewidth]{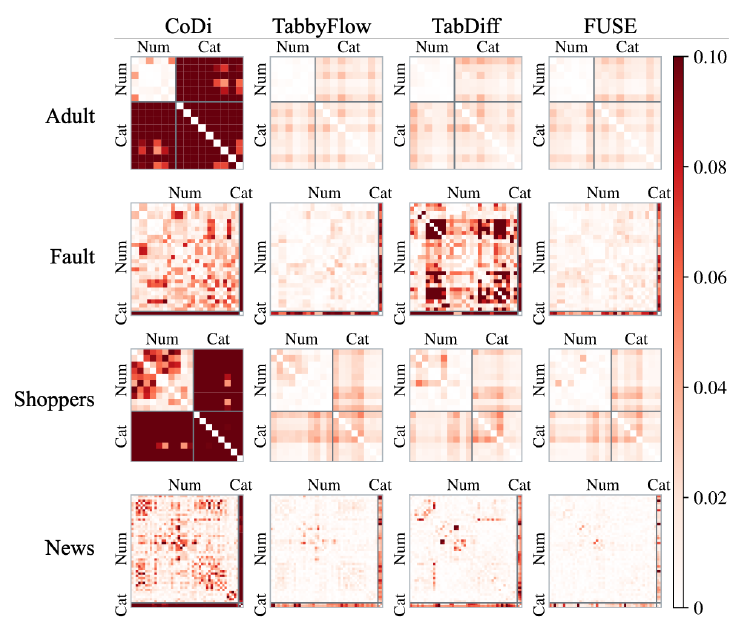}
    \caption{Absolute differences between pairwise correlations from real and synthetic data. Rows correspond to datasets and columns to methods. Values closer to zero indicate more faithful preservation of feature correlation.}
    \label{fig_dependence_heatmap}
\end{figure}

\subsection{Main Results}
\label{sec:baseline_comparison}

Table~\ref{tab:baseline-detailed} compares FUSE with seven tabular data generation baselines in terms of fidelity, pairwise dependence, and downstream utility.

\paragraph{Shape and Trend.}
FUSE demonstrates the most consistent marginal fidelity among the evaluated methods. It ranks first on six datasets and second on the remaining two in Shape, resulting in an average rank of $1.25$. The next-best method, TabbyFlow, obtains an average rank of $2.50$. The improvement is therefore observed across all eight datasets rather than being driven by a small subset of favorable cases.

For Trend, FUSE achieves the highest score on five datasets and the highest cross-dataset mean of $0.982$, compared with $0.978$ for TabbyFlow. The largest improvements over TabbyFlow occur on Diabetes and Fault, where Trend increases from $0.936$ to $0.966$ and from $0.978$ to $0.985$, respectively. TabbyFlow nevertheless obtains a slightly better average rank of $1.75$, compared with $2.00$ for FUSE, reflecting its stronger relative rankings on Beijing, Magic, and News. These results show that FUSE provides a clear improvement in marginal fidelity while remaining highly competitive in preserving pairwise feature dependencies.

\paragraph{C2ST.}
FUSE achieves the best C2ST average rank of $1.62$, followed by TabbyFlow at $1.75$. It ranks first on Default, Beijing, and News and second on each of the remaining five datasets, making it the only method that places among the top two across all eight datasets. Since the C2ST classifier processes all features jointly, this consistency cannot be attributed solely to accurate univariate marginals. Rather, it indicates that FUSE leaves fewer multivariate discrepancies detectable by the selected XGBoost classifier.

\paragraph{Machine Learning Efficiency.}
FUSE also achieves the best average MLE rank of $2.62$, outperforming TabbyFlow at $2.75$ and TabDiff at $3.12$. It obtains the best downstream performance on Shoppers, News, and Fault and the second-best result on Default. The improvement is particularly achieved on News, where FUSE reduces RMSE from $0.872$ for TabbyFlow, the strongest competing method on this dataset, to $0.841$. This corresponds to a relative reduction of approximately $3.6\%$. The gains are not uniform across all tasks, as TabbyFlow achieves a lower RMSE on Beijing and TabSyn obtains a higher AUC on Diabetes. Nevertheless, the best average rank shows that the improvements in distributional fidelity are accompanied by competitive downstream utility.

Overall, FUSE achieves average ranks of $1.25$, $2.00$, $1.62$, and $2.62$ for Shape, Trend, C2ST, and MLE, respectively. The results show a favorable balance across marginal, pairwise, multivariate, and task-oriented evaluations, without relying on uniformly dominant performance under any single metric.

\begin{table*}[t]
\centering
\small
\renewcommand{\arraystretch}{1.08}
\setlength{\textwidth}{1.0mm}

\begin{tabular}{lcccrrrrrrr}
\toprule
& \multicolumn{3}{c}{Architecture}
& \multicolumn{5}{c}{Fidelity}
& \multicolumn{2}{c}{MLE} \\
\cmidrule(lr){2-4}
\cmidrule(lr){5-9}
\cmidrule(lr){10-11}
Variant
& \shortstack{Numerical\\mixture}
& \shortstack{Categorical\\mixture}
& \shortstack{Joint\\attention}
& \shortstack{Shape\\$(\uparrow)$}
& \shortstack{Trend\\$(\uparrow)$}
& \shortstack{C2ST\\$(\uparrow)$}
& \shortstack{$\alpha$-Prec.\\$(\uparrow)$}
& \shortstack{$\beta$-Rec.\\$(\uparrow)$}
& \shortstack{AUC\\$(\uparrow)$}
& \shortstack{RMSE\\$(\downarrow)$} \\
\midrule
Dense FFN
& -- & -- & --
& 0.984 & 0.960 & 0.980 & 0.970 & 0.344
& 0.630 & 0.820 \\

Restricted attention
& $\checkmark$ & $\checkmark$ & --
& 0.984 & 0.960 & \underline{0.990} & 0.972 & 0.371
& 0.644 & 0.818 \\

Shared processing
& -- & -- & $\checkmark$
& \underline{0.984} & \underline{0.976} & 0.985
& \underline{0.985} & 0.482
& \textbf{0.888} & 0.742 \\

Numerical mixture
& $\checkmark$ & -- & $\checkmark$
& 0.984 & 0.973 & 0.983 & 0.985 & \textbf{0.548}
& \underline{0.884} & 0.734 \\

Categorical mixture
& -- & $\checkmark$ & $\checkmark$
& 0.983 & 0.970 & 0.973 & 0.982 & 0.507
& 0.882 & \underline{0.725} \\

\rowcolor[gray]{0.94}
FUSE
& $\checkmark$ & $\checkmark$ & $\checkmark$
& \textbf{0.985} & \textbf{0.977} & \textbf{0.991}
& \textbf{0.986} & \underline{0.545}
& 0.882 & \textbf{0.725} \\
\bottomrule
\end{tabular}
\label{tab:component-analysis}
\caption{Component analysis averaged across eight datasets. Check marks indicate active components. MLE reports classification AUC and regression RMSE. Bold and underlined values denote the best and second-best results based on unrounded means.}
\label{tab:component-analysis}
\end{table*}

\begin{figure}[t]
    \centering
    \includegraphics[width=1.0\linewidth]{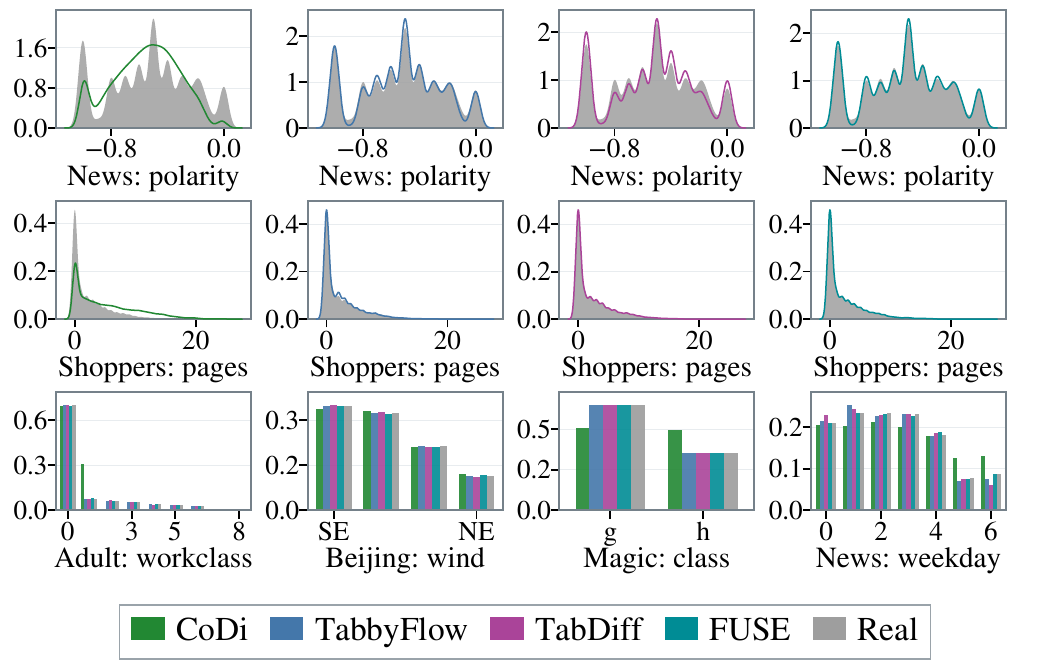}
    \caption{Real and synthetic marginal distributions for numerical and categorical features.}
    \label{fig:marginal_comparison}
\end{figure}

\subsection{Synthetic Data Visualization}
\label{sec_synthetic_visualization}

Figure~\ref{fig:marginal_comparison} compares marginal distributions generated by CoDi, TabbyFlow, TabDiff, and FUSE, while Figure~\ref{fig_dependence_heatmap} visualizes errors in pairwise dependence preservation. 

For numerical features, kernel density estimates show that FUSE recovers both the global shape and local structure of the empirical distributions. Specifically, FUSE captures the sequence of local modes in the polarity variable on News and Shoppers. For categorical variables, FUSE closely matches the empirical category frequencies across Adult, Beijing, Magic, and News. It preserves both dominant and low-frequency classes in Adult and accurately reproduces the target balance in Magic. FUSE also generates a tighter representation of less frequent weekday categories. These results demonstrate that FUSE effectively preserves both dominant categorical structures and minority probability masses.

Figure~\ref{fig_dependence_heatmap} reports the absolute difference between the pairwise correlation matrices of the real and synthetic data. The block-structured matrices reveal how well each method preserves cross-type dependencies. FUSE achieves consistently low correlation errors across all blocks, with a pronounced advantage on News and Fault. Specifically on News, FUSE yields near-zero errors, whereas CoDi exhibits widespread discrepancies and TabDiff retains localized high errors. Ultimately, these visualizations confirm that FUSE successfully captures both the intra-type correlations and the complex cross-type dependencies inherent to mixed-type tabular data.

\subsection{Component Analysis}
\label{sec:component_analysis}

Table~\ref{tab:component-analysis} evaluates six configurations to separate the contributions of adaptive mixture processing and joint attention. All configurations use identical training objectives, data splits, and five matched random seeds. Dense processing replaces the adaptive mixture modules with feed-forward networks (FFN), while restricted attention limits interactions to features of the same type. The shared-processing configuration combines dense networks with joint attention and serves as the reference for evaluating adaptive mixture processing under unrestricted conditioning.

The effect of joint attention is examined through two comparisons. Under dense processing, enabling joint attention preserves Shape at $0.984$ while increasing Trend from $0.960$ to $0.976$ and AUC from $0.630$ to $0.888$. It also reduces RMSE from $0.820$ to $0.742$. A similar pattern appears when both adaptive mixture modules are active. Relative to restricted attention, FUSE improves Trend from $0.960$ to $0.977$, AUC from $0.644$ to $0.882$, and RMSE from $0.818$ to $0.725$, while maintaining comparable marginal fidelity. This result shows that cross-type information exchange is particularly important for dependency preservation and downstream utility. It is consistent with Proposition~\ref{prop:conditioning_penalty}, which characterizes the cost of withholding cross-feature information.

The contribution of adaptive mixture processing is evaluated under joint attention. Compared with shared processing, FUSE improves Shape from $0.984$ to $0.985$, Trend from $0.976$ to $0.977$, C2ST from $0.985$ to $0.991$, and $\beta$-Recall from $0.482$ to $0.545$. It also reduces RMSE from $0.742$ to $0.725$, although AUC decreases slightly from $0.888$ to $0.882$. Consequently, FUSE achieves the best result on four of the five fidelity metrics and the second-best $\beta$-Recall. Applying the mixture module to a single feature type produces metric-specific trade-offs. The numerical-only configuration attains the highest $\beta$-Recall and second-highest AUC, whereas the categorical-only configuration matches the RMSE of FUSE but performs less favorably on the fidelity metrics.

Overall, the results support complementary roles for the two components. Joint attention enables effective cross-type conditioning, while adaptive mixture processing further improves fidelity and coverage through type-specific computation. Their combination provides the strongest aggregate fidelity profile, although no configuration uniformly dominates every metric.
\section{Conclusion}
\label{sec:conclusion}

In this work, we have introduced FUSE, a network architecture for mixed-type variational flow matching that makes the parameter sharing configuration between mixed-type endpoint factors explicit. FUSE combines type-specific adaptive mixture processing with joint attention. The adaptive mixture modules form feature-dependent combinations of shared specialized subnetworks, while joint attention preserves cross-type information exchange and unrestricted conditioning for every endpoint predictor. Our theoretical analysis characterizes the complementary roles of these components. The conditioning information bound quantifies the excess population risk caused by restricted conditioning contexts, while the architectural bottleneck analysis distinguishes information loss from limitations in shared computation. The distributional error bound further relates endpoint-prediction risk to Wasserstein generation error under the stated assumptions. Comprehensive experiments show that FUSE achieves strong aggregate performance across the evaluated metrics. The component analysis provides evidence for joint attention and shows that adaptive mixture processing improves several fidelity and utility measures. 

\newpage
\bibliography{references}

@inproceedings{lipman2023flow,
  author    = {Lipman, Yaron and Chen, Ricky T. Q. and Ben-Hamu, Heli and Nickel, Maximilian and Le, Matt},
  title     = {Flow Matching for Generative Modeling},
  booktitle = {The Eleventh International Conference on Learning Representations},
  year      = {2023},
  url       = {https://openreview.net/forum?id=PqvMRDCJT9t}
}

@article{albergo2023stochastic,
  author  = {Michael Albergo and Nicholas M. Boffi and Eric Vanden-Eijnden},
  title   = {Stochastic Interpolants: A Unifying Framework for Flows and Diffusions},
  journal = {Journal of Machine Learning Research},
  year    = {2025},
  volume  = {26},
  number  = {209},
  pages   = {1--80},
  url     = {http://jmlr.org/papers/v26/23-1605.html}
}

@inproceedings{eijkelboom2024vfm,
  series = {NeurIPS 2024},
  title = {Variational Flow Matching for Graph Generation},
  url = {http://dx.doi.org/10.52202/079017-0374},
  DOI = {10.52202/079017-0374},
  booktitle = {Advances in Neural Information Processing Systems 37},
  publisher = {Neural Information Processing Systems Foundation,  Inc. (NeurIPS)},
  author = {Eijkelboom,  Floor and Bartosh,  Grigory and Naesseth,  Christian and Welling,  Max and Meent,  Jan-Willem Van De},
  year = {2024},
  pages = {11735--11764},
  collection = {NeurIPS 2024}
}

@inproceedings{guzman2025efvfm,
  title     = {Exponential Family Variational Flow Matching for Tabular Data Generation},
  author    = {Guzm{\'a}n-Cordero, Andr{\'e}s and Eijkelboom, Floor and Van De Meent, Jan-Willem},
  booktitle = {Proceedings of the 42nd International Conference on Machine Learning},
  pages     = {21516--21529},
  year      = {2025},
  volume    = {267},
  series    = {Proceedings of Machine Learning Research},
  publisher = {PMLR},
  url       = {https://proceedings.mlr.press/v267/guzman-cordero25a.html}
}

@inproceedings{xu2019ctgan,
 author = {Xu, Lei and Skoularidou, Maria and Cuesta-Infante, Alfredo and Veeramachaneni, Kalyan},
 booktitle = {Advances in Neural Information Processing Systems},
 publisher = {Curran Associates, Inc.},
 title = {Modeling Tabular data using Conditional GAN},
 url = {https://proceedings.neurips.cc/paper_files/paper/2019/file/254ed7d2de3b23ab10936522dd547b78-Paper.pdf},
 volume = {32},
 year = {2019}
}

@inproceedings{zhao2021ctabgan,
    title =  {CTAB-GAN: Effective Table Data Synthesizing},  
    author = {Zhao, Zilong and Kunar, Aditya and Birke, Robert and Chen, Lydia Y.},  
    booktitle =  {Proceedings of The 13th Asian Conference on Machine Learning},  
    pages =  {97--112},  
    year =  {2021},  
    volume =  {157},  
    series =  {Proceedings of Machine Learning Research},  
    month =  {17--19 Nov},  
    publisher =  {PMLR},  
    url =  {https://proceedings.mlr.press/v157/zhao21a.html}
}

@inproceedings{qian2023synthcity,
  series = {NeurIPS 2023},
  title = {Synthcity: a benchmark framework for diverse use cases of tabular synthetic data},
  url = {http://dx.doi.org/10.52202/075280-0140},
  doi = {10.52202/075280-0140},
  booktitle = {Advances in Neural Information Processing Systems 36},
  publisher = {Neural Information Processing Systems Foundation, Inc. (NeurIPS)},
  author = {Qian, Zhaozhi and Davis, Rob and {van der Schaar}, Mihaela},
  year = {2023},
  pages = {3173--3188},
  collection = {NeurIPS 2023}
}

@inproceedings{vanbreugel2023testing,
  series = {NeurIPS 2023},
  title = {Can You Rely on Your Model Evaluation? Improving Model Evaluation with Synthetic Test Data},
  url = {http://dx.doi.org/10.52202/075280-0091},
  doi = {10.52202/075280-0091},
  booktitle = {Advances in Neural Information Processing Systems 36},
  publisher = {Neural Information Processing Systems Foundation, Inc. (NeurIPS)},
  author = {{van Breugel}, Boris and Seedat, Nabeel and Imrie, Fergus and {van der Schaar}, Mihaela},
  year = {2023},
  pages = {1889--1904},
  collection = {NeurIPS 2023}
}

@inproceedings{kotelnikov2023tabddpm,
  title = {{TabDDPM}: Modelling Tabular Data with Diffusion Models},
  author =  {Kotelnikov, Akim and Baranchuk, Dmitry and Rubachev, Ivan and Babenko, Artem},
  booktitle = {Proceedings of the 40th International Conference on Machine Learning},
  pages = {17564--17579},
  year = {2023},
  volume = {202},
  series = {Proceedings of Machine Learning Research},
  month = {23--29 Jul},
  publisher = {PMLR},
  url = 	 {https://proceedings.mlr.press/v202/kotelnikov23a.html}
}

@inproceedings{lee2023codi,
  title = {{CoDi}: Co-evolving Contrastive Diffusion Models for Mixed-type Tabular Synthesis},
  author = {Lee, Chaejeong and Kim, Jayoung and Park, Noseong},
  booktitle = {Proceedings of the 40th International Conference on Machine Learning},
  pages = {18940--18956},
  year = {2023},
  volume = {202},
  series = {Proceedings of Machine Learning Research},
  month = {23--29 Jul},
  publisher =  {PMLR},
  url = {https://proceedings.mlr.press/v202/lee23i.html}
}

@inproceedings{zhang2024tabsyn,
 author = {Zhang, Hengrui and Zhang, Jiani and Shen, Zhengyuan and Srinivasan, Balasubramaniam and Qin, Xiao and Faloutsos, Christos and Rangwala, Huzefa and Karypis, George},
 booktitle = {International Conference on Learning Representations},
 pages = {52829--52857},
 title = {Mixed-Type Tabular Data Synthesis with Score-based Diffusion in Latent Space},
 url = {https://proceedings.iclr.cc/paper_files/paper/2024/file/e9750610639c3e7a849cff746bf60dbd-Paper-Conference.pdf},
 volume = {2024},
 year = {2024}
}

@inproceedings{mueller2025cdtd,
 author = {Mueller, Markus and Gruber, Kathrin and Fok, Dennis},
 booktitle = {International Conference on Learning Representations},
 pages = {39186--39225},
 title = {Continuous Diffusion for Mixed-Type Tabular Data},
 url = {https://proceedings.iclr.cc/paper_files/paper/2025/file/61f4e5747b1b753cb35546b15d981f76-Paper-Conference.pdf},
 volume = {2025},
 year = {2025}
}

@inproceedings{shi2025tabdiff,
 author = {Shi, Juntong and Xu, Minkai and Hua, Harper and Zhang, Hengrui and Ermon, Stefano and Leskovec, Jure},
 booktitle = {International Conference on Learning Representations},
 pages = {37353--37375},
 title = {{TabDiff}: a Mixed-type Diffusion Model for Tabular Data Generation},
 url = {https://proceedings.iclr.cc/paper_files/paper/2025/file/5c882988ce5fac487974ee4f415b96a9-Paper-Conference.pdf},
 volume = {2025},
 year = {2025}
}

@inproceedings{zhang2025tabnat,
  title = {{TabNAT}: A Continuous-Discrete Joint Generative Framework for Tabular Data},
  author = {Zhang, Hengrui and Fang, Liancheng and Wu, Qitian and Yu, Philip S.},
  booktitle = {Proceedings of the 42nd International Conference on Machine Learning},
  pages =  {74858--74881},
  year =  {2025},
  volume = {267},
  series = {Proceedings of Machine Learning Research},
  month = {13--19 Jul},
  publisher = {PMLR},
  url = {https://proceedings.mlr.press/v267/zhang25t.html}
}

@article{si2026tabrep,
    title={{TabRep}: Training Tabular Diffusion Models with a Simple and Effective Continuous Representation},
    author={Jacob Si and Zijing Ou and Mike Qu and Zhengrui Xiang and Yingzhen Li},
    journal={Transactions on Machine Learning Research},
    issn={2835-8856},
    year={2026},
    url={https://openreview.net/forum?id=yRbtFEh2OP}
}

@inproceedings{jolicoeurmartineau2024forest,
  title = {Generating and Imputing Tabular Data via Diffusion and Flow-based Gradient-Boosted Trees},
  author = {Jolicoeur-Martineau, Alexia and Fatras, Kilian and Kachman, Tal},
  booktitle = {Proceedings of The 27th International Conference on Artificial Intelligence and Statistics},
  pages = {1288--1296},
  year = {2024},
  volume = {238},
  series = {Proceedings of Machine Learning Research},
  month = {02--04 May},
  publisher = {PMLR},
  url = {https://proceedings.mlr.press/v238/jolicoeur-martineau24a.html}
}

@inproceedings{puigcerver2024softmoe,
     author = {Puigcerver, Joan and Riquelme Ruiz, Carlos and Mustafa, Basil and Houlsby, Neil},
     booktitle = {International Conference on Learning Representations},
     pages = {28435--28445},
     title = {From Sparse to Soft Mixtures of Experts},
     url = {https://proceedings.iclr.cc/paper_files/paper/2024/file/79fea214543ba263952ac3f4e5452b14-Paper-Conference.pdf},
     volume = {2024},
     year = {2024}
}

@book{ambrosio2008gradient,
  author    = {Ambrosio, Luigi and Gigli, Nicola and Savar{\'e}, Giuseppe},
  title     = {Gradient Flows: In Metric Spaces and in the Space of Probability Measures},
  series    = {Lectures in Mathematics ETH Z{\"u}rich},
  edition   = {2},
  publisher = {Birkh{\"a}user Basel},
  year      = {2008},
  isbn      = {978-3-7643-8722-8},
  doi       = {10.1007/978-3-7643-8722-8},
  url       = {https://doi.org/10.1007/978-3-7643-8722-8}
}

@inproceedings{bao2022vlmo,
  series = {NeurIPS 2022},
  title = {{VLMo}: Unified Vision-Language Pre-Training with Mixture-Of-Modality-Experts},
  url = {http://dx.doi.org/10.52202/068431-2384},
  DOI = {10.52202/068431-2384},
  booktitle = {Advances in Neural Information Processing Systems 35},
  publisher = {Neural Information Processing Systems Foundation,  Inc. (NeurIPS)},
  author = {Bao,  Hangbo and Wang,  Wenhui and Dong,  Li and Liu,  Qiang and Mohammed,  Owais Khan and Aggarwal,  Kriti and Som,  Subhojit and Piao,  Songhao and Wei,  Furu},
  year = {2022},
  pages = {32897–-32912},
  collection = {NeurIPS 2022}
}

@inproceedings{esser2024scaling,
  title = {Scaling Rectified Flow Transformers for High-Resolution Image Synthesis},
  author = {Esser, Patrick and Kulal, Sumith and Blattmann, Andreas and Entezari, Rahim and M\"{u}ller, Jonas and Saini, Harry and Levi, Yam and Lorenz, Dominik and Sauer, Axel and Boesel, Frederic and Podell, Dustin and Dockhorn, Tim and English, Zion and Rombach, Robin},
  booktitle = {Proceedings of the 41st International Conference on Machine Learning},
  pages = {12606--12633},
  year = {2024},
  volume = {235},
  series = {Proceedings of Machine Learning Research},
  month = {21--27 Jul},
  publisher = {PMLR},
  url = {https://proceedings.mlr.press/v235/esser24a.html}
}

@article{nasution2026flow,
    title={Flow Matching for Tabular Data Synthesis},
    author={Bahrul Ilmi Nasution and Floor Eijkelboom and Mark Elliot and Richard Allmendinger and Christian A. Naesseth},
    journal={Transactions on Machine Learning Research},
    issn={2835-8856},
    year={2026},
    url={https://openreview.net/forum?id=RdOjoAa66L}
}

@inproceedings{mueller2026cascaded,
    title={Cascaded Flow Matching for Heterogeneous Tabular Data with Mixed-Type Features},
    author={Markus Mueller and Kathrin Gruber and Dennis Fok},
    booktitle={Forty-third International Conference on Machine Learning},
    year={2026},
    url={https://openreview.net/forum?id=l2ywV9sV0L}
}

@article{choi2026tabgeoflow,
  title = {{TabGeoFlow}: A Geometric Flow Matching Model for Tabular Data Synthesis},
  volume = {40},
  ISSN = {2159-5399},
  url = {http://dx.doi.org/10.1609/aaai.v40i25.39192},
  DOI = {10.1609/aaai.v40i25.39192},
  number = {25},
  journal = {Proceedings of the AAAI Conference on Artificial Intelligence},
  publisher = {Association for the Advancement of Artificial Intelligence (AAAI)},
  author = {Choi,  Jong In},
  year = {2026},
  month = Mar,
  pages = {20562-–20569}
}

@inproceedings{holzmuller2024better,
  series = {NeurIPS 2024},
  title = {Better by default: Strong pre-tuned {MLP}s and boosted trees on tabular data},
  url = {http://dx.doi.org/10.52202/079017-0837},
  DOI = {10.52202/079017-0837},
  booktitle = {Advances in Neural Information Processing Systems 37},
  publisher = {Neural Information Processing Systems Foundation,  Inc. (NeurIPS)},
  author = {Holzm\"{u}ller,  David and Grinsztajn,  Léo and Steinwart,  Ingo},
  year = {2024},
  pages = {26577--26658},
  collection = {NeurIPS 2024}
}

@inproceedings{gardner2023benchmarking,
  series = {NeurIPS 2023},
  title = {Benchmarking Distribution Shift in Tabular Data with TableShift},
  url = {http://dx.doi.org/10.52202/075280-2324},
  DOI = {10.52202/075280-2324},
  booktitle = {Advances in Neural Information Processing Systems 36},
  publisher = {Neural Information Processing Systems Foundation,  Inc. (NeurIPS)},
  author = {Gardner,  Josh and Popovic,  Zoran and Schmidt,  Ludwig},
  year = {2023},
  pages = {53385--53432},
  collection = {NeurIPS 2023}
}

@inproceedings{vaswani2017attention,
     author = {Vaswani, Ashish and Shazeer, Noam and Parmar, Niki and Uszkoreit, Jakob and Jones, Llion and Gomez, Aidan N and Kaiser, \L ukasz and Polosukhin, Illia},
     booktitle = {Advances in Neural Information Processing Systems},
     publisher = {Curran Associates, Inc.},
     title = {Attention is All you Need},
     url = {https://proceedings.neurips.cc/paper_files/paper/2017/file/3f5ee243547dee91fbd053c1c4a845aa-Paper.pdf},
     volume = {30},
     year = {2017}
}

@inproceedings{puigcerver2024soft,
     author = {Puigcerver, Joan and Riquelme Ruiz, Carlos and Mustafa, Basil and Houlsby, Neil},
     booktitle = {International Conference on Learning Representations},
     pages = {28435--28445},
     title = {From Sparse to Soft Mixtures of Experts},
     url = {https://proceedings.iclr.cc/paper_files/paper/2024/file/79fea214543ba263952ac3f4e5452b14-Paper-Conference.pdf},
     volume = {2024},
     year = {2024}
}

@article{little2025synthetic,
  title={Synthetic census microdata generation: A comparative study of synthesis methods examining the trade-off between disclosure risk and utility},
  author={Little, Claire and Allmendinger, Richard and Elliot, Mark},
  journal={Journal of Official Statistics},
  volume={41},
  number={1},
  pages={255--308},
  year={2025},
  publisher={SAGE Publications Sage UK: London, England}
}

@inproceedings{
albergo2023building,
title={Building Normalizing Flows with Stochastic Interpolants},
author={Michael Samuel Albergo and Eric Vanden-Eijnden},
booktitle={The Eleventh International Conference on Learning Representations },
year={2023}
}

@inproceedings{
liu2023flow,
title={Flow Straight and Fast: Learning to Generate and Transfer Data with Rectified Flow},
author={Xingchao Liu and Chengyue Gong and Qiang Liu},
booktitle={The Eleventh International Conference on Learning Representations },
year={2023}
}

@InProceedings{alaa2022faithful,
  title = 	 {How Faithful is your Synthetic Data? {S}ample-level Metrics for Evaluating and Auditing Generative Models},
  author =       {Alaa, Ahmed and Van Breugel, Boris and Saveliev, Evgeny S. and van der Schaar, Mihaela},
  booktitle = 	 {Proceedings of the 39th International Conference on Machine Learning},
  pages = 	 {290--306},
  year = 	 {2022}
}

@inproceedings{shazeer2017outrageously,
  author = {Noam Shazeer and
            Azalia Mirhoseini and
            Krzysztof Maziarz and
            Andy Davis and
            Quoc V. Le and
            Geoffrey E. Hinton and
            Jeff Dean},
  title = {Outrageously Large Neural Networks: The Sparsely-Gated
           Mixture-of-Experts Layer},
  booktitle = {5th International Conference on Learning Representations,
               ICLR 2017, Toulon, France, April 24--26, 2017,
               Conference Track Proceedings},
  publisher = {OpenReview.net},
  year = {2017},
  url = {https://openreview.net/forum?id=B1ckMDqlg}
}

\section*{Appendix}

\appendix
\section{Proofs}
\label{app:proofs}

This section follows the order of the theoretical analysis in the main paper. We first prove Proposition~1, then derive the two architectural bottleneck examples, and finally prove Theorem~1. We conclude by verifying that a fixed FUSE network satisfies the regularity conditions used in Theorem~1.

\subsection{Proof of Proposition~1}

We use the population version of the endpoint objective in Equation~(3). For
$f=(\mu,\{\pi_k\}_{k=1}^{d_{\mathrm{cat}}})$, write
\begin{align*}
\mathcal R_T(f)
={}&
\mathbb E\!\left[
\frac{\|X_1^{\mathrm{num}}-\mu_t(X_t)\|_2^2}{2\nu_t}
\right] \\
&-
\sum_{k=1}^{d_{\mathrm{cat}}}
\mathbb E\!\left[
\log \pi_{k,t}(X_t)_{C_1^{(k)}}
\right],
\end{align*}
where the expectation includes $t$, whose density on $[0,T]$ is $\rho$. Each $\pi_{k,t}(X_t)$ is required to lie in the probability simplex $\Delta_{K_k}$. For a conditioning $\sigma$-algebra $\mathcal A$, let $\mathcal R_T^*(\mathcal A)$ be the infimum of $\mathcal R_T(f)$ over all $\mathcal A$-measurable predictors.

The assumptions of Proposition~1 do not require $\nu_t$ to be bounded away from zero. Consequently, both Bayes risks can be infinite, in which case the expression $+\infty-(+\infty)$ is not defined. We therefore use a common truncation only to define the difference in this exceptional case. Set $a(t)=(2\nu_t)^{-1}$ and $a_M(t)=a(t)\wedge M$, and let $\mathcal R_{T,M}^*(\mathcal A)$ denote the Bayes risk obtained by replacing $a$ with $a_M$ in the numerical term. We interpret the risk gap in Proposition~1 as
\begin{align}
\Delta_T(\mathcal G,\mathcal F)
:={}&
\sup_{M>0}
\left\{
\mathcal R_{T,M}^*(\mathcal G)
-
\mathcal R_{T,M}^*(\mathcal F)
\right\}.
\label{eq:supp-common-truncation-gap}
\end{align}
The proof below shows that the quantity inside the supremum is nondecreasing
in $M$. Whenever both untruncated Bayes risks are finite,
$\Delta_T(\mathcal G,\mathcal F)$ agrees with the ordinary difference
$\mathcal R_T^*(\mathcal G)-\mathcal R_T^*(\mathcal F)$.

\begin{proof}[Proof of Proposition~1]
For $\mathcal A\in\{\mathcal G,\mathcal F\}$, define
\begin{align*}
m_{\mathcal A}
&=
\mathbb E[X_1^{\mathrm{num}}\mid\mathcal A],
&
 p_{k,\mathcal A}(r)
&=
\mathbb P(C_1^{(k)}=r\mid\mathcal A).
\end{align*}
Because $\mathbb E\|X_1^{\mathrm{num}}\|_2^2<\infty$, conditional Jensen's
inequality gives $m_{\mathcal A}\in L^2$. For any finite
$\mathcal A$-measurable numerical predictor $h$, the conditional projection
identity gives, in $[0,\infty]$,
\begin{align*}
\mathbb E\!\left[
\|X_1^{\mathrm{num}}-h\|_2^2
\mid\mathcal A
\right]
={}&
\mathbb E\left[
\|X_1^{\mathrm{num}}-m_{\mathcal A}\|_2^2
\mid\mathcal A
\right] \\
&+
\|m_{\mathcal A}-h\|_2^2.
\end{align*}
The identity for a predictor without an a priori square-integrability
assumption follows by truncating $h$ and applying monotone convergence.
Since $a_M(t)$ is $\sigma(t)$-measurable and
$\sigma(t)\subseteq\mathcal A$, multiplication by $a_M(t)$ and integration
show that $m_{\mathcal A}$ minimizes the truncated numerical risk.

For the categorical term, let
$H(p)=-\sum_r p(r)\log p(r)$. For any $\mathcal A$-measurable probability
vector $q_k$,
\begin{align*}
\mathbb E\left[
-\log q_k(C_1^{(k)})
\mid\mathcal A
\right]
={}&
H(p_{k,\mathcal A})
+
\operatorname{KL}
\left(
 p_{k,\mathcal A}
 \middle\|
 q_k
\right).
\end{align*}
Hence the Bayes categorical predictor is $p_{k,\mathcal A}$. Combining the
numerical and categorical terms yields
\begin{align}
\mathcal R_{T,M}^*(\mathcal A)
={}&
\mathbb E\!\left[
 a_M(t)
 \|X_1^{\mathrm{num}}-m_{\mathcal A}\|_2^2
\right] 
+
\sum_{k=1}^{d_{\mathrm{cat}}}
\mathbb E\!\left[
 H(p_{k,\mathcal A})
\right].
\label{eq:supp-truncated-bayes-risk}
\end{align}

Because $\mathcal G\subseteq\mathcal F$, the tower property gives
$m_{\mathcal G}=\mathbb E[m_{\mathcal F}\mid\mathcal G]$ and
$p_{k,\mathcal G}=\mathbb E[p_{k,\mathcal F}\mid\mathcal G]$ coordinatewise.
Applying the nested projection identity to the numerical term in
Equation~\eqref{eq:supp-truncated-bayes-risk} gives
\begin{align*}
&
\mathbb E\!\left[
 a_M(t)
 \left\{
 \|X_1^{\mathrm{num}}-m_{\mathcal G}\|_2^2
 -
 \|X_1^{\mathrm{num}}-m_{\mathcal F}\|_2^2
 \right\}
\right]
\notag\\
&\qquad=
\mathbb E\!\left[
 a_M(t)
 \|m_{\mathcal F}-m_{\mathcal G}\|_2^2
\right].
\end{align*}
For each categorical feature, the cross-entropy decomposition gives
\begin{align*}
\mathbb E\left[
 H(p_{k,\mathcal G})
-
 H(p_{k,\mathcal F})
\right]=
\mathbb E\left[
\operatorname{KL}
\left(
 p_{k,\mathcal F}
 \middle\|
 p_{k,\mathcal G}
\right)
\right].
\end{align*}
Therefore, for every $M>0$,
\begin{align*}
&
\mathcal R_{T,M}^*(\mathcal G)
-
\mathcal R_{T,M}^*(\mathcal F)
\notag\\
&\qquad=
\mathbb E\!\left[
 a_M(t)
 \|m_{\mathcal F}-m_{\mathcal G}\|_2^2
\right]
+
\sum_{k=1}^{d_{\mathrm{cat}}}
\mathbb E\!\left[
\operatorname{KL}
\left(
 p_{k,\mathcal F}
 \middle\|
 p_{k,\mathcal G}
\right)
\right].
\end{align*}
The right-hand side is nondecreasing in $M$. Taking the supremum in
Equation~\eqref{eq:supp-common-truncation-gap} and applying monotone
convergence gives
\begin{align*}
\Delta_T(\mathcal G,\mathcal F)
={}&
\mathbb E\!\left[
\frac{
\|m_{\mathcal F}-m_{\mathcal G}\|_2^2
}{2\nu_t}
\right]
\notag\\
&+
\sum_{k=1}^{d_{\mathrm{cat}}}
\mathbb E\!\left[
\operatorname{KL}
\left(
 p_{k,\mathcal F}
 \middle\|
 p_{k,\mathcal G}
\right)
\right]
\geq 0.
\end{align*}
This is the identity stated in Proposition~1, with the common-truncation
interpretation used only when ordinary subtraction is undefined.

Finally, $a(t)>0$ almost surely, and every KL divergence is nonnegative.
Thus the gap is zero if and only if
\begin{align*}
m_{\mathcal F}=m_{\mathcal G}
\quad\text{and}\quad
p_{k,\mathcal F}=p_{k,\mathcal G}
\quad
\text{for every }k,
\end{align*}
almost surely. This proves the equality characterization.
\end{proof}

We briefly record when the common-truncation convention reduces to ordinary
subtraction. Since each categorical entropy is bounded by $\log K_k$, both
untruncated Bayes risks are finite whenever
\begin{align*}
\mathbb E\!\left[
\frac{\|X_1^{\mathrm{num}}\|_2^2}{\nu_t}
\right]
<\infty.
\end{align*}
Indeed, the zero numerical predictor and uniform categorical predictors then
have finite risk. For the variance schedule used in the experiments,
$\nu_t=1-t^2\geq0$ on $[0,1]$, so the assumed second moment of
$X_1^{\mathrm{num}}$ is sufficient. Hence the difference in Proposition~1 is
an ordinary finite subtraction for the reported experimental settings.

\subsection{Derivations for the Two Architectural Bottlenecks}

\paragraph{Conditioning bottleneck.}
Recall the construction in the main paper:
\begin{align*}
\mathbb P(S=1)
=
\mathbb P(S=-1)
=
\tfrac12,\;
U
=
\delta S+\xi,\;
\xi \sim \mathcal N(0,\sigma^2),
\end{align*}
where $\xi\perp S$, $\delta\neq0$, and $\sigma^2>0$. Bayes' rule gives
\begin{align*}
\log
\frac{\mathbb P(S=1\mid U=u)}
{\mathbb P(S=-1\mid U=u)}
&=
\frac{(u+\delta)^2-(u-\delta)^2}{2\sigma^2}
=
\frac{2\delta u}{\sigma^2}.
\end{align*}
Since $S\in\{-1,1\}$, it follows that
\begin{align*}
\mathbb E[S\mid U=u]
&=
\tanh\!\left(\frac{\delta u}{\sigma^2}\right),\\
\operatorname{Var}(S\mid U=u)
&=
\operatorname{sech}^2\!\left(\frac{\delta u}{\sigma^2}\right).
\end{align*}
For the numerical endpoint mean
$m^{(1)}(U,S)=\alpha U+cS$, the full and restricted posterior means are
\begin{align*}
m_{\mathcal F}^{(1)}
&=
\alpha U+cS,
&
m_{\mathcal G}^{(1)}
&=
\alpha U+c\mathbb E[S\mid U],
\end{align*}
where $\mathcal F=\sigma(U,S)$ and $\mathcal G=\sigma(U)$. Proposition~1
therefore gives
\begin{align*}
\mathbb E\!\left[
\frac{
\{m_{\mathcal F}^{(1)}-m_{\mathcal G}^{(1)}\}^2
}{2\nu}
\right]
&=
\frac{c^2}{2\nu}
\mathbb E\!\left[
\{S-\mathbb E[S\mid U]\}^2
\right]
\notag\\
&=
\frac{c^2}{2\nu}
\mathbb E\!\left[
\operatorname{Var}(S\mid U)
\right]
\notag\\
&=
\frac{c^2}{2\nu}
\mathbb E\!\left[
\operatorname{sech}^2\!\left(
\frac{\delta U}{\sigma^2}
\right)
\right]
>0.
\end{align*}
The strict inequality follows because $c\neq0$, $\nu>0$, and
$\operatorname{sech}^2(z)>0$ for every finite $z$.

\paragraph{Representation bottleneck.}
Let $\mathcal H=L^2(P_{W,V})$. Since $W$ and $V$ are independent standard
normal variables,
\begin{align*}
\langle W,W\rangle_{\mathcal H}
&=
\langle V,V\rangle_{\mathcal H}
=1,
&
\langle W,V\rangle_{\mathcal H}
&=0.
\end{align*}
Fix $h\in\mathcal H$ with $\|h\|_{\mathcal H}=1$. The least-squares
coefficients are
\begin{align*}
\theta_2^*
&=
\langle aW,h\rangle_{\mathcal H},
&
\theta_3^*
&=
\langle bV,h\rangle_{\mathcal H},
\end{align*}
and hence
\begin{align*}
&
\inf_{\theta_2,\theta_3\in\mathbb R}
\mathbb E\!\left[
(aW-\theta_2h)^2
+
(bV-\theta_3h)^2
\right]
\notag\\
&\qquad=
 a^2+b^2
-
 a^2\langle W,h\rangle_{\mathcal H}^2
-
 b^2\langle V,h\rangle_{\mathcal H}^2.
\end{align*}
By Bessel's inequality,
\begin{align*}
a^2\langle W,h\rangle_{\mathcal H}^2
+
b^2\langle V,h\rangle_{\mathcal H}^2
&\leq
\max\{a^2,b^2\}
\left(
\langle W,h\rangle_{\mathcal H}^2
+
\langle V,h\rangle_{\mathcal H}^2
\right)
\notag\\
&\leq
\max\{a^2,b^2\}.
\end{align*}
The bound is attained by $h=W$ when $a^2\geq b^2$ and by $h=V$ otherwise.
Therefore,
\begin{align*}
\inf_{(f_2,f_3)\in\mathcal H_{\mathrm{shared}}}
\mathbb E\!\left[
\{m^{(2)}-f_2\}^2
+
\{m^{(3)}-f_3\}^2 
\right]\\
=
\min\{a^2,b^2\}.
\end{align*}

Two shared latent components are sufficient once the features can use
different recombination weights. Define
\begin{align*}
g_1
&=
\frac{3aW-bV}{2},
&
g_2
&=
\frac{-aW+3bV}{2},
\notag\\
c_2
&=
(3/4,1/4),
&
c_3
&=
(1/4,3/4).
\end{align*}
Both $c_2$ and $c_3$ are probability vectors, and direct calculation gives
\begin{align*}
c_2^\top(g_1,g_2)
&=
 aW,
&
c_3^\top(g_1,g_2)
&=
 bV.
\end{align*}
Thus the same component dictionary represents both endpoint means exactly,
while feature-dependent recombination selects a different mixture for each
feature.

\subsection{Proof of Theorem~1}

For $t<T$, let $p_t=\mathcal L(X_t)$, where $X_t$ is the interpolation in
Equation~(1). Define the true and learned embedded endpoint means by
\begin{align*}
m_t(x)
&=
\mathbb E[X_1\mid X_t=x],
\notag\\
\widehat m_{\theta,t}(x)
&=
\bigl(
\mu_{\theta,t}(x),
\pi_{\theta,1,t}(x),
\ldots,
\pi_{\theta,d_{\mathrm{cat}},t}(x)
\bigr)
\in\mathbb R^D.
\end{align*}
Each categorical probability vector occupies its corresponding one-hot block.
As defined in Section~2, the endpoint predictor induces
\begin{align*}
v_{\theta,t}(x)
&=
\frac{\widehat m_{\theta,t}(x)-x}{1-t}.
\end{align*}
Let $f^*$ be the unrestricted Bayes predictor and use the shorthand
\begin{align*}
\mathcal E_T(\theta)
&=
\mathcal R_T(f_\theta)-\mathcal R_T(f^*),
\end{align*}
which is the excess risk appearing in Theorem~1.

\begin{proof}[Proof of Theorem~1]
Because $\nu_t\geq\underline\nu>0$, the zero numerical predictor and uniform
categorical predictors give
\begin{align*}
0
\leq
\mathcal R_T(f^*)
&\leq
\frac{
\mathbb E\|X_1^{\mathrm{num}}\|_2^2
}{2\underline\nu}
+
\sum_{k=1}^{d_{\mathrm{cat}}}\log K_k
<\infty.
\end{align*}
Thus $\mathcal E_T(\theta)$ is well defined in $[0,\infty]$. If it is
infinite, the claimed bound is immediate, so we assume
$\mathcal E_T(\theta)<\infty$ below. Moreover, each categorical one-hot block
has squared norm one, and therefore
\begin{align*}
\mathbb E\|X_1\|_2^2
&=
\mathbb E\|X_1^{\mathrm{num}}\|_2^2
+d_{\mathrm{cat}}
<\infty.
\end{align*}
Hence $p_0,p_1\in\mathcal P_2(\mathbb R^D)$.

For categorical feature $k$, let
\begin{align*}
p_{k,t}^*(x)
&=
\mathbb P(C_1^{(k)}=\cdot\mid X_t=x).
\end{align*}
The categorical blocks of $m_t(x)$ are precisely
$p_{k,t}^*(x)$. The conditional projection and cross-entropy identities used
in the proof of Proposition~1 yield
\begin{align*}
\mathcal E_T(\theta)
={}&
\frac12
\int_0^T
\frac{\rho(t)}{\nu_t}
\|\mu_{\theta,t}-m_t^{\mathrm{num}}\|_{L^2(p_t)}^2
\,dt
\notag\\
&+
\sum_{k=1}^{d_{\mathrm{cat}}}
\int_0^T
\rho(t)
\int
\operatorname{KL}
\left(
 p_{k,t}^*(x)
 \middle\|
 \pi_{\theta,k,t}(x)
\right)
 p_t(dx)
\,dt.
\end{align*}
Pinsker's inequality implies
\begin{align*}
\|p_{k,t}^*(x)-\pi_{\theta,k,t}(x)\|_2^2
&\leq
\|p_{k,t}^*(x)-\pi_{\theta,k,t}(x)\|_1^2
\notag\\
&\leq
2\operatorname{KL}
\left(
 p_{k,t}^*(x)
 \middle\|
 \pi_{\theta,k,t}(x)
\right).
\end{align*}
Set
\begin{align*}
b_t^2
&=
\|\widehat m_{\theta,t}-m_t\|_{L^2(p_t)}^2.
\end{align*}
Jointly measurable versions of the conditional means can be chosen below, so
$t\mapsto b_t^2$ is measurable. Combining the numerical projection identity
with Pinsker's inequality gives
\begin{align}
\int_0^T\rho(t)b_t^2\,dt
&\leq
2\max\{\overline\nu,1\}
\mathcal E_T(\theta).
\label{eq:supp-risk-mean-calibration}
\end{align}
This is the only point at which the numerical Gaussian loss and the
categorical cross-entropy loss enter the Wasserstein argument.

The assumed uniform $L$-Lipschitz property gives the linear-growth bound
\begin{align*}
\|v_{\theta,t}(x)\|_2
&\leq
\|v_{\theta,t}(0)\|_2+L\|x\|_2.
\end{align*}
Together with joint measurability and
$\int_0^T\|v_{\theta,t}(0)\|_2\,dt<\infty$, this gives a unique global
Carath\'eodory flow $\Phi_t$. Gronwall's inequality yields
\begin{align*}
\|\Phi_t(x)\|_2
&\leq
 e^{Lt}
\left(
\|x\|_2
+
\int_0^t\|v_{\theta,s}(0)\|_2\,ds
\right).
\end{align*}
Since $p_0\in\mathcal P_2(\mathbb R^D)$, the learned law
$\widehat p_t^\theta=(\Phi_t)_\#p_0$ also belongs to
$\mathcal P_2(\mathbb R^D)$ for every $t\leq T$.

We next construct an oracle velocity whose marginal law is $p_t$. Let
$\gamma_\sigma$ denote the density of $\mathcal N(0,\sigma^2I_D)$. Since
$T<1$, for $t\in[0,T]$ define
\begin{align}
d_t(x)
&=
\int
\gamma_{1-t}(x-ty)
\,p_1(dy),
\notag\\
n_t(x)
&=
\int
 y\gamma_{1-t}(x-ty)
\,p_1(dy),
\notag\\
m_t(x)
&=
\frac{n_t(x)}{d_t(x)},
\notag\\
u_t(x)
&=
\frac{m_t(x)-x}{1-t}
=
\mathbb E[X_1-X_0\mid X_t=x].
\label{eq:supp-oracle-velocity}
\end{align}
The denominator $d_t(x)$ is strictly positive. Because
$\mathbb E\|X_1\|_2<\infty$ and $1-t\geq1-T>0$, parameterized dominated
convergence shows that $(t,x)\mapsto(d_t(x),n_t(x))$ is Borel measurable.
Thus Equation~\eqref{eq:supp-oracle-velocity} provides a jointly Borel
version of both $m_t$ and $u_t$.

For every $\varphi\in C_c^1(\mathbb R^D)$, differentiating the explicit
interpolation in Equation~(1) gives, for almost every $t$,
\begin{align*}
\frac{d}{dt}
\int\varphi(x)\,p_t(dx)
&=
\mathbb E\!\left[
\nabla\varphi(X_t)^\top(X_1-X_0)
\right]
\notag\\
&=
\int
\nabla\varphi(x)^\top u_t(x)
\,p_t(dx).
\end{align*}
Moreover, conditional Jensen's inequality gives
\begin{align*}
\int_0^T
\|u_t\|_{L^2(p_t)}^2
\,dt
&\leq
T\mathbb E\|X_1-X_0\|_2^2
<\infty,
\end{align*}
and the interpolation coupling gives
\begin{align*}
W_2(p_s,p_t)
&\leq
|s-t|
\left(
\mathbb E\|X_1-X_0\|_2^2
\right)^{1/2}.
\end{align*}
Therefore $(p_t,u_t)$ satisfies the continuity equation with finite kinetic
energy. The superposition principle
\citep[Theorem~8.2.1]{ambrosio2008gradient} provides a process
$(\Gamma_t)_{t\in[0,T]}$ on a common probability space such that
\begin{align*}
\Gamma_t
&\sim
p_t,
&
\dot\Gamma_t
&=
 u_t(\Gamma_t)
\end{align*}

Initialize the learned flow with the same random variable $\Gamma_0$ and set
\begin{align*}
\widehat\Gamma_t
&=
\Phi_t(\Gamma_0),
&
D_t
&=
\|\widehat\Gamma_t-\Gamma_t\|_2,
\notag\\
e_t
&=
\|v_{\theta,t}(\Gamma_t)-u_t(\Gamma_t)\|_2.
\end{align*}
By the endpoint-induced form of $v_{\theta,t}$ and
Equation~\eqref{eq:supp-oracle-velocity},
\begin{align*}
\mathbb E[e_t^2]
&=
\frac{b_t^2}{(1-t)^2}.
\end{align*}
Equation~\eqref{eq:supp-risk-mean-calibration}, the lower bound on $\rho$,
and $1-t\geq1-T$ imply
\begin{align*}
\mathbb E\!\left[
\int_0^T e_t^2\,dt
\right]
&\leq
\frac{
2\max\{\overline\nu,1\}
}{
\underline\rho(1-T)^2
}
\mathcal E_T(\theta)
<\infty.
\end{align*}
Thus $e_\cdot\in L^1(0,T)$ almost surely. Using the $L$-Lipschitz property of
$v_{\theta,t}$, the chain rule gives, for almost every $t$,
\begin{align*}
D_t'
&\leq
LD_t+e_t,
\qquad
D_0=0.
\end{align*}
The same inequality holds at times for which $D_t=0$ by the standard
absolute-continuity argument. Gronwall's inequality therefore gives
\begin{align*}
D_T
&\leq
\int_0^T
 e^{L(T-t)}e_t
\,dt.
\end{align*}
Using the coupling $(\widehat\Gamma_T,\Gamma_T)$, Minkowski's inequality,
Cauchy--Schwarz, and
$\Gamma_L(T)=\int_0^T e^{2L(T-t)}\,dt$, we obtain
\begin{align}
W_2(\widehat p_T^\theta,p_T)
&\leq
\|D_T\|_{L^2}
\notag\\
&\leq
\int_0^T
\frac{e^{L(T-t)}}{1-t}
 b_t
\,dt
\notag\\
&\leq
\frac{
\sqrt{K_{\rho,\nu}\Gamma_L(T)}
}{1-T}
\sqrt{\mathcal E_T(\theta)},
\label{eq:supp-learned-oracle-bound}
\end{align}
where
$K_{\rho,\nu}=2\max\{\overline\nu,1\}/\underline\rho$, exactly as in
Theorem~1.

Finally, coupling $X_T=(1-T)X_0+TX_1$ with $X_1$ gives
\begin{align}
W_2(p_T,p_1)
&\leq
(1-T)
\left(
\mathbb E\|X_0-X_1\|_2^2
\right)^{1/2}.
\label{eq:supp-terminal-truncation}
\end{align}
The triangle inequality, together with
Equations~\eqref{eq:supp-learned-oracle-bound} and
\eqref{eq:supp-terminal-truncation}, proves Theorem~1.
\end{proof}

Theorem~1 concerns the continuous preprocessed Euclidean state in
$\mathbb R^D$, including the one-hot categorical blocks. It therefore bounds
the distribution transported by the ODE before categorical argmax and inverse
preprocessing.

\paragraph{Relation to uniform training on $[0,1]$.}
Theorem~1 allows a general time density $\rho$ on $[0,T]$. When training uses
$t\sim\operatorname{Unif}(0,1)$ as in Section~2, let $r_\theta(t)\geq0$ denote
the fixed-time excess endpoint risk. The uniform risk restricted to $[0,T]$
satisfies
\begin{align*}
\mathcal E_T^{\mathrm{unif}}(\theta)
&=
\frac1T\int_0^T r_\theta(t)\,dt
\leq
\frac1T\int_0^1 r_\theta(t)\,dt
=
\frac1T\mathcal E_1^{\mathrm{unif}}(\theta).
\end{align*}
This relation explains how the population objective trained on the full time
interval controls the truncated-time risk used in the theorem.

\subsection{Regularity of the FUSE Vector Field}

\begin{lemma}
\label{lem:regularity_FUSE}
Fix $T\in(0,1)$, and let $\theta$ be a finite parameter. Then there exist
constants $B_{\theta,T},K_{\theta,T}<\infty$ such that
\begin{align}
\sup_{0\leq t\leq T}
\|\widehat m_{\theta,t}(0)\|_2
&\leq
B_{\theta,T},
\notag\\
\|\widehat m_{\theta,t}(x)-\widehat m_{\theta,t}(y)\|_2
&\leq
K_{\theta,T}\|x-y\|_2
\quad 
\label{eq:supp-fuse-endpoint-regularity}
\end{align}
for all $t\in[0,T]$, $x,y\in\mathbb R^D$. 
\end{lemma}

\begin{proof}
We use the Euclidean norm for vectors, the Frobenius norm for token matrices, and $\|\cdot\|_{\mathrm{op}}$ for the operator norm. Assuming the schema and architecture are finite, all spaces appearing below are finite-dimensional. All constants may depend on the fixed architecture, the parameter setting $\theta$, and $T$, but not on the input $x$.

We first establish uniform boundedness and Lipschitz continuity of the
normalization operators used in the network. We then propagate these
properties through the adaptive-mixture networks, joint-attention, and endpoint heads.

\paragraph{Time-conditioned normalization.}
Let
\begin{align*}
P
&=
I_{d_h}
-
\frac{1}{d_h}\mathbf 1\mathbf 1^\top
\end{align*}
be the centering matrix. For $\epsilon_{\mathrm{LN}}>0$, layer
normalization has the form
\begin{align*}
\operatorname{LN}_{\gamma,\beta}(z)
&=
\gamma\odot
\frac{Pz}{
\left\{
d_h^{-1}\|Pz\|_2^2+\epsilon_{\mathrm{LN}}
\right\}^{1/2}
}
+\beta.
\end{align*}
Because
\begin{align*}
\frac{\|Pz\|_2}{
\left\{
d_h^{-1}\|Pz\|_2^2+\epsilon_{\mathrm{LN}}
\right\}^{1/2}
}
&\leq
\sqrt{d_h},
\end{align*}
we obtain
\begin{align*}
\sup_{z\in\mathbb R^{d_h}}
\left\|
\operatorname{LN}_{\gamma,\beta}(z)
\right\|_2
&\leq
\sqrt{d_h}\|\gamma\|_\infty+\|\beta\|_2
=:B_{\mathrm{LN}}.
\end{align*}

The positive value of $\epsilon_{\mathrm{LN}}$ prevents the normalization denominator from approaching zero. The Jacobian of the normalization map is therefore uniformly bounded, and hence
\begin{align*}
\left\|
\operatorname{LN}_{\gamma,\beta}(z)
-
\operatorname{LN}_{\gamma,\beta}(z')
\right\|_2
&\leq
\frac{\|\gamma\|_\infty}{
\sqrt{\epsilon_{\mathrm{LN}}}
}
\|z-z'\|_2
\\
&=:K_{\mathrm{LN}}\|z-z'\|_2.
\end{align*}

Each processing or attention branch uses its own normalization. We write a generic such operator as
\begin{align*}
\mathcal N_q(t,z)
&=
a_q(t)\odot
\operatorname{LN}_{\gamma_q,\beta_q}(z)
+
c_q(t).
\end{align*}

For the implementation in Equation~(9), $a_q(t)=1+\gamma_q(t)$ and $c_q(t)=\zeta_q(t)$. Since the time embedding and all affine maps are continuous in $t$, compactness of $[0,T]$ gives
\begin{align*}
A_q
&:=
\sup_{0\leq t\leq T}\|a_q(t)\|_\infty
<\infty,
&
C_q
&:=
\sup_{0\leq t\leq T}\|c_q(t)\|_2
<\infty.
\end{align*}
It follows that
\begin{align*}
\sup_{\substack{0\leq t\leq T\\z\in\mathbb R^{d_h}}}
\|\mathcal N_q(t,z)\|_2
&\leq
A_qB_{\mathrm{LN},q}+C_q
=:B_{\mathcal N,q},
\\
\|\mathcal N_q(t,z)-\mathcal N_q(t,z')\|_2
&\leq
A_qK_{\mathrm{LN},q}\|z-z'\|_2
\\
&=:K_{\mathcal N,q}\|z-z'\|_2.
\end{align*}

Applying the same normalization to a token matrix preserves Lipschitz continuity in the Frobenius norm. Since each schema has finitely many tokens, there are finite constants $\overline B_{\mathcal N,q}$ and $\overline K_{\mathcal N,q}$ satisfying
\begin{align*}
\sup_{t,H}\|\mathcal N_q(t,H)\|_{\mathrm F}
&\leq
\overline B_{\mathcal N,q},
\\
\|\mathcal N_q(t,H)-\mathcal N_q(t,H')\|_{\mathrm F}
&\leq
\overline K_{\mathcal N,q}
\|H-H'\|_{\mathrm F}.
\end{align*}

\paragraph{Adaptive-mixture processing.}
Consider a mixture of networks for feature type $r\in\{\mathrm{num},\mathrm{cat}\}$ at a fixed layer. To simplify notation, omit the layer index and write
\begin{align*}
\overline H_r
&=
\mathcal N_r^{\mathrm{mix}}(t,H_r),
\\
R_r
&=
\frac{\overline H_r\Phi_r}{\sqrt{d_h}},
\\
D_r
&=
\operatorname{softmax}_{\mathrm{col}}(R_r),
&
C_r
&=
\operatorname{softmax}_{\mathrm{row}}(R_r).
\end{align*}
The two softmax operations correspond to aggregation across feature tokens and recombination across latent components in Equations~(4)--(5) of the main paper.

The normalization bounds imply
\begin{align*}
\|R_r\|_{\mathrm F}
&\leq
\frac{
\overline B_{\mathcal N,r}
\|\Phi_r\|_{\mathrm{op}}
}{
\sqrt{d_h}
},
\\
\|R_r-R_r'\|_{\mathrm F}
&\leq
\frac{
\overline K_{\mathcal N,r}
\|\Phi_r\|_{\mathrm{op}}
}{
\sqrt{d_h}
}
\|H_r-H_r'\|_{\mathrm F}.
\end{align*}
Softmax on a fixed finite-dimensional axis is bounded and Lipschitz. Therefore, finite constants $B_{D,r},B_{C,r},K_{D,r}$, and $K_{C,r}$ exist such that
\begin{align*}
\|D_r\|_{\mathrm F}
&\leq B_{D,r},  \quad
\|C_r\|_{\mathrm F}
\leq B_{C,r},
\\
\|D_r-D_r'\|_{\mathrm F}
&\leq
K_{D,r}\|H_r-H_r'\|_{\mathrm F},
\\
\|C_r-C_r'\|_{\mathrm F}
&\leq
K_{C,r}\|H_r-H_r'\|_{\mathrm F}.
\end{align*}
The aggregated components are $U_r=D_r^\top\overline H_r.$
They are uniformly bounded because
\begin{align*}
\|U_r\|_{\mathrm F}
&\leq 
\|D_r\|_{\mathrm F}
\|\overline H_r\|_{\mathrm F}
\\
&\leq
B_{D,r}\overline B_{\mathcal N,r}
=:B_{U,r}.
\end{align*}
For two inputs $H_r$ and $H_r'$, we have
\begin{align*}
U_r-U_r'
&=
D_r^\top(\overline H_r-\overline H_r')
+
(D_r-D_r')^\top\overline H_r'.
\end{align*}
Hence,
\begin{align*}
\|U_r-U_r'\|_{\mathrm F}
&\leq
\left(
B_{D,r}\overline K_{\mathcal N,r}
+
K_{D,r}\overline B_{\mathcal N,r}
\right)
\|H_r-H_r'\|_{\mathrm F}
\\
&=:K_{U,r}\|H_r-H_r'\|_{\mathrm F}.
\end{align*}

All components of $U_r$ lie in the compact set
\begin{align*}
\mathcal K_{U,r}
&=
\left\{
u:\|u\|_2\leq B_{U,r}
\right\}.
\end{align*}
Every specialized subnetwork is a fixed finite composition of affine maps and continuously differentiable activations. Since there are finitely many subnetworks, continuity on $\mathcal K_{U,r}$ gives
\begin{align*}
B_{g,r}
&:=
\max_m
\sup_{u\in\mathcal K_{U,r}}
\|g_{r,m}(u)\|_2
<\infty,
\\
K_{g,r}
&:=
\max_m
\sup_{u\in\mathcal K_{U,r}}
\|Dg_{r,m}(u)\|_{\mathrm{op}}
<\infty.
\end{align*}
If $V_r$ denotes the matrix obtained by stacking the transformed components, the finite number of components implies that there are constants $B_{V,r},K_{V,r}<\infty$ satisfying
\begin{align*}
\|V_r\|_{\mathrm F}
&\leq
B_{V,r},
\\
\|V_r-V_r'\|_{\mathrm F}
&\leq
K_{V,r}\|H_r-H_r'\|_{\mathrm F}.
\end{align*}

The adaptive-mixture residual network is
\begin{align*}
\mathcal B_r^{\mathrm{mix}}(t,H_r)
&=
C_rV_r.
\end{align*}
Using
\begin{align*}
C_rV_r-C_r'V_r'
&=
C_r(V_r-V_r')
+
(C_r-C_r')V_r',
\end{align*}
we obtain
\begin{align*}
\|\mathcal B_r^{\mathrm{mix}}(t,H_r)\|_{\mathrm F}
&\leq
B_{C,r}B_{V,r}
=:b_r^{\mathrm{mix}},
\\
\|\mathcal B_r^{\mathrm{mix}}(t,H_r)
- &
\mathcal B_r^{\mathrm{mix}}(t,H_r')\|_{\mathrm F}
\\&\leq
\left(
B_{C,r}K_{V,r}
+
K_{C,r}B_{V,r}
\right)
\|H_r-H_r'\|_{\mathrm F}
\\
&=:
\ell_r^{\mathrm{mix}}
\|H_r-H_r'\|_{\mathrm F}.
\end{align*}
Thus every adaptive-mixture residual network is uniformly bounded and globally Lipschitz.

\paragraph{Joint attention.}
Let $H$ denote the concatenated numerical and categorical tokens entering a joint-attention network, and define
\begin{align*}
Z
&=
\mathcal N^{\mathrm{att}}(t,H).
\end{align*}
The normalization bounds provide finite constants $B_Z,K_Z$ such that
\begin{align*}
\|Z\|_{\mathrm F}
&\leq B_Z,
\\
\|Z-Z'\|_{\mathrm F}
&\leq K_Z\|H-H'\|_{\mathrm F}.
\end{align*}

For one attention head, write
\begin{align*}
Q
&=
ZW_Q+\mathbf 1b_Q^\top,
&
K
&=
ZW_K+\mathbf 1b_K^\top,
&
V
&=
ZW_V+\mathbf 1b_V^\top.
\end{align*}
Because these maps are affine and $Z$ is bounded, there exist finite constants $B_Q,B_K,B_V,K_Q,K_K,K_V$ such that
\begin{align*}
\|Q\|_{\mathrm F}
\leq B_Q,\;
\|K\|_{\mathrm F}
&\leq B_K, \;
\|V\|_{\mathrm F}
\leq B_V,
\\
\|Q-Q'\|_{\mathrm F}
&\leq K_Q\|H-H'\|_{\mathrm F},
\\
\|K-K'\|_{\mathrm F}
&\leq K_K\|H-H'\|_{\mathrm F},
\\
\|V-V'\|_{\mathrm F}
&\leq K_V\|H-H'\|_{\mathrm F}.
\end{align*}

Let
\begin{align*}
A
&=
\operatorname{softmax}_{\mathrm{row}}
\left(
\frac{QK^\top}{\sqrt{d_{\mathrm{head}}}}
\right)
\end{align*}
be the attention-weight matrix. The difference between two attention matrices satisfies
\begin{align*}
QK^\top-Q'K'^\top
&=
(Q-Q')K^\top
+
Q'(K-K')^\top.
\end{align*}
Consequently,
\begin{align*}
\|QK^\top-Q'K'^\top\|_{\mathrm F}
&\leq
\left(
K_QB_K+B_QK_K
\right)
\|H-H'\|_{\mathrm F}.
\end{align*}
Since row-wise softmax is bounded and Lipschitz on a fixed finite-dimensional space, there exist $B_A,K_A<\infty$ such that
\begin{align*}
\|A\|_{\mathrm F}
&\leq B_A,
\\
\|A-A'\|_{\mathrm F}
&\leq K_A\|H-H'\|_{\mathrm F}.
\end{align*}
The head output $O=AV$ therefore satisfies
\begin{align*}
\|O\|_{\mathrm F}
&\leq
B_AB_V,
\\
\|O-O'\|_{\mathrm F}
&\leq
\left(
B_AK_V+K_AB_V
\right)
\|H-H'\|_{\mathrm F}.
\end{align*}

A finite concatenation of attention heads followed by a finite affine output projection preserves boundedness and global Lipschitz continuity. Thus the complete multi-head attention network in Equation~(6) of the main paper satisfies
\begin{align*}
\sup_{t,H}
\|\mathcal B^{\mathrm{att}}(t,H)\|_{\mathrm F}
&\leq
b^{\mathrm{att}}
<\infty,
\\
\|\mathcal B^{\mathrm{att}}(t,H)
-
\mathcal B^{\mathrm{att}}(t,H')\|_{\mathrm F}
&\leq
\ell^{\mathrm{att}}
\|H-H'\|_{\mathrm F}
\end{align*}
for finite constants $b^{\mathrm{att}}$ and
$\ell^{\mathrm{att}}$.

\paragraph{Propagation through the residual network.}
Flatten the finite collection of token matrices into a single Euclidean vector and index all adaptive-mixture and attention residual networks in their computational order by $q=1,\ldots,Q$. Write
\begin{align*}
H_q(t,x)
&=
H_{q-1}(t,x)
+
\mathcal B_q
\left(
t,H_{q-1}(t,x)
\right).
\end{align*}
The preceding arguments show that finite constants $b_q,\ell_q$ exist such that
\begin{align*}
\sup_{t,H}\|\mathcal B_q(t,H)\|_2
&\leq b_q,
\\
\|\mathcal B_q(t,H)-\mathcal B_q(t,H')\|_2
&\leq
\ell_q\|H-H'\|_2.
\end{align*}

The initial tokenization is affine in $x$ and depends on the time embedding. Compactness of $[0,T]$ therefore gives
\begin{align*}
B_0
&:=
\sup_{0\leq t\leq T}
\|H_0(t,0)\|_2
<\infty,
\\
K_0
&:=
\sup_{0\leq t\leq T}
\operatorname{Lip}_x H_0(t,\cdot)
<\infty.
\end{align*}
At the origin,
\begin{align*}
\|H_q(t,0)\|_2
&\leq
\|H_{q-1}(t,0)\|_2+b_q.
\end{align*}
Induction over the finite residual sequence gives
\begin{align*}
\sup_{0\leq t\leq T}
\|H_Q(t,0)\|_2
&\leq
B_0+\sum_{q=1}^Q b_q.
\end{align*}

For arbitrary $x,y\in\mathbb R^D$,
\begin{align*}
\|H_q(t,x)&-H_q(t,y)\|_2
\\&\leq
\|H_{q-1}(t,x)-H_{q-1}(t,y)\|_2
\\
&\quad+
\left\|
\mathcal B_q(t,H_{q-1}(t,x))
-
\mathcal B_q(t,H_{q-1}(t,y))
\right\|_2
\\
&\leq
(1+\ell_q)
\|H_{q-1}(t,x)-H_{q-1}(t,y)\|_2.
\end{align*}
A second induction yields
\begin{align*}
\|H_Q(t,x)-H_Q(t,y)\|_2
&\leq
K_0
\prod_{q=1}^Q(1+\ell_q)
\|x-y\|_2.
\end{align*}

\paragraph{Endpoint heads.}
The endpoint heads consist of layer normalization with a positive epsilon, finite affine maps, and categorical softmax maps. They therefore define a Lipschitz output map with a finite value at the origin. Hence, there are finite constants $B_{\mathrm{out}}$ and $K_{\mathrm{out}}$ such that
\begin{align*}
\|\widehat m_{\theta,t}(0)\|_2
&\leq
B_{\mathrm{out}}
+
K_{\mathrm{out}}
\|H_Q(t,0)\|_2,
\\
\|\widehat m_{\theta,t}(x)
-
\widehat m_{\theta,t}(y)\|_2
&\leq
K_{\mathrm{out}}
\|H_Q(t,x)-H_Q(t,y)\|_2.
\end{align*}
Therefore, we may take
\begin{align*}
B_{\theta,T}
&=
B_{\mathrm{out}}
+
K_{\mathrm{out}}
\left(
B_0+\sum_{q=1}^Q b_q
\right),
\\
K_{\theta,T}
&=
K_{\mathrm{out}}K_0
\prod_{q=1}^Q(1+\ell_q).
\end{align*}
These constants satisfy
\begin{align*}
\sup_{0\leq t\leq T}
\|\widehat m_{\theta,t}(0)\|_2
&\leq
B_{\theta,T},
\\
\|\widehat m_{\theta,t}(x)
-
\widehat m_{\theta,t}(y)\|_2
&\leq
K_{\theta,T}\|x-y\|_2.
\end{align*}

Every operation used above is continuous in its finite-dimensional arguments, and the time embedding is continuous on $[0,T]$. Since FUSE is a finite composition of these operations, $(t,x)\mapsto\widehat m_{\theta,t}(x)$ is jointly continuous.

Finally, the endpoint predictor induces
\begin{align*}
v_{\theta,t}(x)
&=
\frac{\widehat m_{\theta,t}(x)-x}{1-t}.
\end{align*}
Since $t\leq T<1$,
\begin{align*}
\|v_{\theta,t}(x)-v_{\theta,t}(y)\|_2
&\leq
\frac{K_{\theta,T}+1}{1-T}
\|x-y\|_2,
\\
\|v_{\theta,t}(0)\|_2
&\leq
\frac{B_{\theta,T}}{1-T},
\\
\int_0^T
\|v_{\theta,t}(0)\|_2\,dt
&\leq
\frac{TB_{\theta,T}}{1-T}
<\infty.
\end{align*}
Joint continuity implies joint Borel measurability. Thus $v_{\theta,t}$ is jointly measurable, uniformly Lipschitz in $x$, and integrable at the origin, as required by Theorem~1.
\end{proof}

\section{Related Work}
\label{sec:related-work}

\paragraph{Mixed-Type Tabular Generation.}
Early approaches such as CTGAN and TVAE combine type-specific preprocessing with adversarial or variational objectives \citep{xu2019ctgan}. Diffusion-based methods represent mixed-type variables in different ways. TabDDPM and CoDi couple continuous and categorical diffusion processes, while TabDiff uses a joint continuous-time model with feature-wise learnable noise schedules \citep{kotelnikov2023tabddpm,lee2023codi,shi2025tabdiff}. TabSyn performs diffusion in a learned continuous latent space, CDTD applies continuous diffusion to both variable types, and TabRep constructs a unified continuous representation \citep{zhang2024tabsyn,mueller2025cdtd,si2026tabrep}. TabNAT instead combines diffusion-based numerical generation with masked categorical generation \citep{zhang2025tabnat}. FUSE follows EF-VFM in representing categorical variables as one-hot vectors along a Euclidean interpolation, while revising the network structure used to predict the endpoint-factor parameters.

\paragraph{Variational Flow Matching.}
Flow matching learns a velocity field that transports a base distribution along a prescribed probability path \citep{lipman2023flow,albergo2023stochastic}. Variational flow matching expresses this field through a variational approximation to the posterior distribution over endpoints \citep{eijkelboom2024vfm}. When the conditional velocity field is linear in the endpoint, the resulting field depends only on the marginal endpoint means. This permits a mean-field factorization without requiring the joint posterior covariance. EF-VFM extends this formulation to mixed-type data using exponential-family endpoint factors and moment matching, with TabbyFlow providing its tabular implementation \citep{guzman2025efvfm}. Other tabular flow methods replace neural predictors with gradient-boosted trees, impose geometric constraints, adopt cascaded generation, or examine the effects of probability paths and sampling schemes \citep{jolicoeurmartineau2024forest,choi2026tabgeoflow, mueller2026cascaded,nasution2026flow}. FUSE retains the probability path and endpoint-factor family of EF-VFM, but focuses on how type-specific processing and full-state conditioning should be organized.

\paragraph{Specialized Neural Processing.} 
Mixture architectures process representations through specialized subnetworks whose contributions depend on the input. Sparse mixture models use learned gates to select a subset of subnetworks, whereas Soft MoE replaces discrete assignment with differentiable dispatch and combination weights over weighted token aggregates \citep{shazeer2017outrageously,puigcerver2024softmoe}. A related architectural pattern separates modality-specific feed-forward processing from cross-modal interaction. VLMo combines modality-specific feed-forward modules with shared self-attention, while multimodal diffusion transformers use separate modality-specific parameters while allowing joint attention \citep{bao2022vlmo,esser2024scaling}. These architectures show that specialized processing can coexist with unrestricted interaction across input types.
\section{Implementation Details}
\label{app:implementation-details}

This section specifies the implementation choices omitted from the main paper. It describes time conditioning, the architectural parameterization, and capacity matching.

\subsection{Time Conditioning}

Let $d_h\geq 2$ denote the hidden dimension and let
$d_\omega=\lfloor d_h/2\rfloor$. The time $t$ is embedded as
\begin{align}
\omega_k
&=
\exp\!\left\{
-\frac{k\log(10000)}{\max\{1,d_\omega\}}
\right\},
\quad
k=0,\ldots,d_\omega-1,
\notag\\
s(t)
&=
\bigl(
\{\sin(t\omega_k)\}_{k=0}^{d_\omega-1},
\{\cos(t\omega_k)\}_{k=0}^{d_\omega-1}
\bigr),
\notag\\
\psi(t)
&=
W_2\operatorname{SiLU}\!\left(W_1s(t)+b_1\right)+b_2.
\label{eq:supp-time-embedding}
\end{align}
When $d_h$ is odd, a zero coordinate is appended to $s(t)$ so that its dimension equals $d_h$. The feature-specific affine maps described in the main paper add a type-specific projection of $\psi(t)$ to each initial token.

Each processing and attention branch uses independent time-conditioned normalization parameters. Let $q\in\{\mathrm{proc},\mathrm{attn}\}$ index the branch. For feature type $r\in\{\mathrm{num},\mathrm{cat}\}$ and layer $\ell$,
\begin{align}
\bigl(
\zeta_{r,q}^{(\ell)}(t),
\gamma_{r,q}^{(\ell)}(t)
\bigr)
&=
W_{r,q}^{(\ell)}
\operatorname{SiLU}\!\left(\psi(t)\right)
+
b_{r,q}^{(\ell)},
\notag\\
\mathcal N_{r,q}^{(\ell)}(H;t)
&=
\operatorname{LN}(H)
\odot
\left\{1+\gamma_{r,q}^{(\ell)}(t)\right\}
+
\zeta_{r,q}^{(\ell)}(t).
\label{eq:supp-adaptive-normalization}
\end{align}
Layer normalization is applied without learned affine parameters and uses $\epsilon_{\mathrm{LN}}=10^{-5}$. The output maps producing $\zeta_{r,q}^{(\ell)}(t)$ and $\gamma_{r,q}^{(\ell)}(t)$ are zero-initialized. The normalized representation used by adaptive mixture processing is
\[
\overline H_{r,t}^{(\ell)}
=
\mathcal N_{r,\mathrm{proc}}^{(\ell)}
\bigl(H_{r,t}^{(\ell)};t\bigr),
\]
while the attention uses its own normalization parameters.

\subsection{Architecture Specification}
\paragraph{Adaptive mixture processing.}
The implementation computes the alignment score matrix as
\begin{align}
\widetilde R_{r,t}^{(\ell)}
&=
\overline H_{r,t}^{(\ell)}\Phi_r^{(\ell)}.
\label{eq:supp-unscaled-score}
\end{align}
Since $\Phi_r^{(\ell)}$ is learned, this omission changes its parameterization but not the represented family of alignment scores. Each specialized subnetwork contains two linear layers with a GELU activation. No dropout or learned residual gate is used within the processing module.

\paragraph{Joint attention.}
Joint attention is applied without an attention mask. The restricted-attention variant retains the same parameterization but masks the two cross-type attention blocks, allowing interactions only within each feature type. Every configuration contains $L=4$ blocks, each consisting of one processing module followed by one attention module. Attention dropout is set to zero.

\paragraph{Endpoint heads.}
After the final block, numerical and categorical tokens are normalized using separate affine layer-normalization modules,
\begin{align*}
h_{r,j,t}^{(L)}
&=
\operatorname{LN}_{r}
\left(
\left[H_{r,t}^{(L)}\right]_{j}
\right),
\quad
r\in\{\mathrm{num},\mathrm{cat}\},
\quad
L=4.
\label{eq:supp-final-token-normalization}
\end{align*}
Thus, $h_{r,j,t}^{(L)}$ in the endpoint-head equations of the main paper denotes the normalized final token rather than the corresponding row of $H_{r,t}^{(L)}$ before normalization. The numerical head applies no additional output activation, while the categorical heads retain the feature-wise parameterization defined in the main paper. The final normalization uses $\epsilon_{\mathrm{LN}}=10^{-5}$. For fixed learned parameters, this positive normalization constant ensures bounded final token representations, as required in Lemma~\ref{lem:regularity_FUSE}.

\begin{table*}[t]
\centering
\small
\setlength{\tabcolsep}{0.9mm}
\begin{tabular}{lrrrrrrrrr}
\toprule
\multicolumn{10}{l}{\textbf{$\alpha$-Precision ($\uparrow$)}} \\
Method & Adult & Default & Beijing & Shoppers & Magic & News & Diabetes & Fault & Avg. rank $\downarrow$ \\
\cmidrule(lr){1-10}
CTGAN & 0.715(.003) & 0.683(.002) & 0.916(.003) & 0.821(.003) & 0.789(.004) & \underline{0.974}(.002) & 0.640(.013) & 0.656(.012) & 6.63 \\
TVAE & 0.923(.003) & 0.828(.004) & 0.859(.003) & 0.567(.004) & 0.951(.003) & 0.857(.003) & 0.774(.018) & 0.878(.014) & 6.50 \\
CoDi & 0.497(.003) & 0.710(.001) & 0.949(.003) & 0.727(.005) & 0.842(.005) & 0.523(.002) & 0.896(.024) & 0.924(.033) & 6.38 \\
TabDDPM & 0.971(.004) & \textbf{0.994}(.002) & 0.977(.002) & 0.938(.004) & \underline{0.992}(.003) & 0.000(.000) & 0.876(.206) & 0.003(.012) & 5.00 \\
TabSyn & 0.974(.003) & 0.924(.003) & 0.987(.003) & 0.978(.006) & 0.991(.003) & 0.968(.003) & 0.955(.025) & 0.936(.018) & 3.63 \\
TabbyFlow & \underline{0.991}(.002) & 0.990(.002) & \underline{0.994}(.002) & \textbf{0.993}(.002) & 0.989(.004) & 0.918(.003) & \textbf{0.972}(.012) & \underline{0.955}(.015) & \underline{2.63} \\
TabDiff & 0.979(.003) & 0.982(.002) & 0.981(.003) & \underline{0.992}(.002) & 0.992(.003) & 0.965(.003) & 0.926(.029) & 0.697(.051) & 3.75 \\
\rowcolor[gray]{0.94}
FUSE & \textbf{0.992}(.002) & \underline{0.993}(.002) & \textbf{0.996}(.002) & 0.985(.005) & \textbf{0.994}(.003) & \textbf{0.983}(.006) & \underline{0.969}(.013) & \textbf{0.976}(.012) & \textbf{1.50} \\
\midrule
\multicolumn{10}{l}{\textbf{$\beta$-Recall ($\uparrow$)}} \\
Method & Adult & Default & Beijing & Shoppers & Magic & News & Diabetes & Fault & Avg. rank $\downarrow$ \\
\cmidrule(lr){1-10}
CTGAN & 0.270(.002) & 0.169(.002) & 0.260(.002) & 0.286(.004) & 0.082(.003) & 0.161(.002) & 0.078(.009) & 0.007(.002) & 7.13 \\
TVAE & 0.340(.003) & 0.185(.001) & 0.170(.002) & 0.173(.003) & 0.328(.003) & 0.309(.004) & 0.346(.010) & 0.113(.007) & 6.50 \\
CoDi & 0.229(.003) & 0.059(.002) & 0.300(.003) & 0.305(.004) & 0.491(.003) & 0.112(.001) & 0.688(.021) & 0.613(.025) & 5.50 \\
TabDDPM & 0.479(.002) & 0.467(.004) & 0.400(.002) & \underline{0.527}(.006) & \textbf{0.523}(.004) & 0.000(.000) & \underline{0.713}(.173) & 0.000(.000) & 4.00 \\
TabSyn & 0.428(.003) & 0.403(.002) & 0.307(.002) & 0.494(.006) & 0.481(.005) & 0.443(.002) & 0.510(.022) & 0.338(.040) & 4.38 \\
TabbyFlow & \textbf{0.548}(.003) & \underline{0.517}(.003) & \textbf{0.552}(.003) & \textbf{0.600}(.005) & 0.473(.004) & \textbf{0.578}(.002) & 0.622(.028) & \underline{0.619}(.017) & \textbf{2.13} \\
TabDiff & \underline{0.538}(.003) & \textbf{0.520}(.002) & 0.472(.002) & 0.487(.005) & \underline{0.493}(.004) & 0.387(.003) & 0.466(.033) & 0.036(.012) & 3.63 \\
\rowcolor[gray]{0.94}
FUSE & 0.497(.004) & 0.452(.004) & \underline{0.551}(.002) & 0.526(.005) & 0.468(.005) & \underline{0.444}(.009) & \textbf{0.767}(.027) & \textbf{0.703}(.028) & \underline{2.75} \\
\bottomrule
\end{tabular}

\caption{$\alpha$-Precision and $\beta$-Recall across eight datasets. Results are reported as means with standard deviations in parentheses.
Bold and underlined values indicate the best and second-best results, respectively, and the proposed method is shaded.}
\label{tab:supp-alpha-beta}
\end{table*}

\subsection{Capacity Matching}

For each feature type, the dense FFN baseline uses a two-layer feed-forward module with hidden width $d_{\mathrm{standard}}=4096$, equal to the sum of the widths of the four specialized subnetworks. Across the eight datasets, FUSE contains between $20.264$ and $20.317$ million trainable parameters, while the dense processing variants contain between $20.242$ and $20.295$ million. The difference is $22{,}528$ parameters, corresponding to less than $0.12\%$ of the total parameter count.

A leading-order multiply-accumulate comparison provides additional context for the processing modules. Let
$
N=d_{\mathrm{num}}+d_{\mathrm{cat}}
$
denote the total number of numerical and categorical feature tokens. In the evaluated configuration,
\[
S_{\mathrm{num}}=S_{\mathrm{cat}}=S=8,
\quad
d_{\mathrm{ff}}=1024,
\quad
d_{\mathrm{std}}=4096.
\]
The per-observation counts for one processing layer are
\begin{align*}
\operatorname{MAC}_{\mathrm{adap}}
&=
4Sd_hd_{\mathrm{ff}}+3NSd_h,
\notag\\
\operatorname{MAC}_{\mathrm{std}}
&=
2Nd_hd_{\mathrm{std}}.
\end{align*}
At the evaluated widths, the ratio
$\operatorname{MAC}_{\mathrm{std}}/
\operatorname{MAC}_{\mathrm{adap}}$
ranges from $2.24$ to $11.59$. These counts include only the leading matrix products in the processing modules. They exclude attention, normalization, activation functions, and other element-wise operations and therefore are not interpreted as end-to-end runtime measurements.
\section{Experimental Details}
\label{app:experimental-details}

This section provides additional details for the experiments in the main paper.

\begin{table}[t]
\centering
\small
\setlength{\tabcolsep}{3.0pt}
\begin{tabular}{llrrrr}
\toprule
Dataset & Task & Train & Test & Num. & Cat. \\
\midrule
Adult & Binary & 32,561 & 16,281 & 6 & 9 \\
Default & Binary & 27,000 & 3,000 & 14 & 10 \\
Beijing & Regression & 37,581 & 4,176 & 7 & 5 \\
Shoppers & Binary & 11,097 & 1,233 & 10 & 8 \\
Magic & Binary & 17,117 & 1,902 & 10 & 1 \\
News & Regression & 35,679 & 3,965 & 46 & 2 \\
Diabetes & Binary & 691 & 77 & 8 & 1 \\
Fault & Multiclass & 1,746 & 195 & 27 & 1 \\
\bottomrule
\end{tabular}
\caption{Downstream tasks, processed train and test split sizes, and variable counts. Binary and multiclass denote classification tasks. Num. and Cat. denote numerical and categorical variables, respectively. The prediction target is included in its corresponding variable type.}
\label{tab:dataset-summary}
\end{table}

\paragraph{Datasets. }
Table~\ref{tab:dataset-summary} summarizes the profiles of the datasets. Adult, Default, Shoppers, Magic, and Diabetes define binary classification tasks, while Fault defines a multiclass classification task. Beijing and News are regression datasets. 

\paragraph{Training. }

All FUSE configurations use hidden dimension $d_h=256$, four layers, and four attention heads. For each feature type, adaptive mixture processing uses $M_r=4$ parallel subnetworks, $P_r=2$ latent components per subnetwork, and hidden width $d_{\mathrm{ff}}=1024$. Training proceeds using AdamW with learning rate $10^{-3}$. The model parameters are maintained using an exponential moving average with decay $0.997$. Batch sizes are $1024$ for Adult, Magic, and News, $672$ for Beijing, and $512$ for the remaining datasets. The numerical endpoint variance is
$
\nu_t
=
1-t^2$ for $t\in[0,1).$
This schedule specifies the variance of the numerical endpoint factor and satisfies $\nu_t\rightarrow0$ as $t\rightarrow1$. 

\paragraph{Sampling. }

\begin{figure*}[t!]
\centering
\includegraphics[width=\textwidth]
{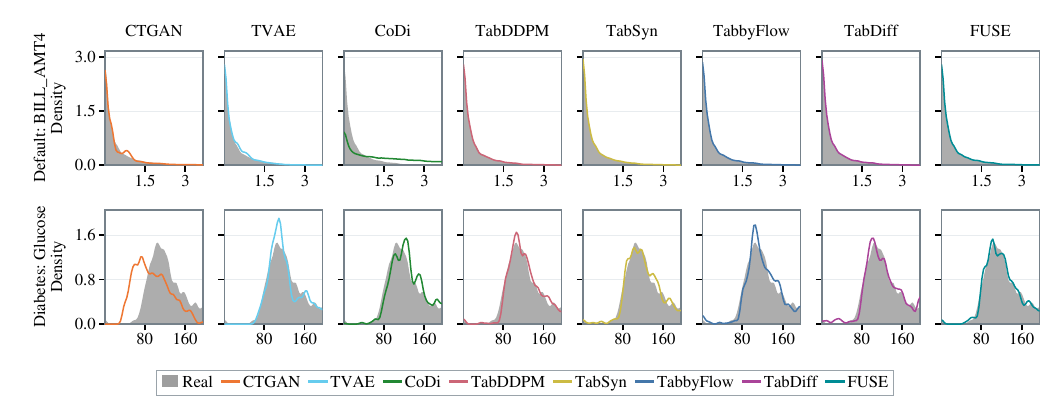}
\caption{Numerical marginal distributions on Default and Diabetes. Gray denotes the real data.}
\label{fig:supp-numerical-marginals}
\end{figure*}

\begin{figure*}[t]
\centering
\includegraphics[width=\textwidth]
{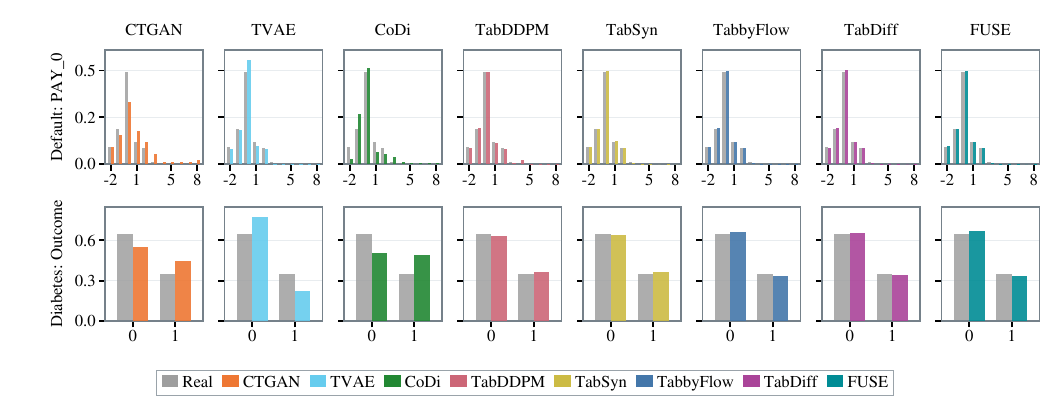}
\caption{Categorical marginals on Default and Diabetes. Gray denotes the real data.}
\label{fig:supp-categorical-marginals}
\end{figure*}

Generation begins from $
\widehat X_0\sim\mathcal N(0,I_D),$
where $D$ is the dimension of the transformed row representation. The resulting ODE is integrated to $t_{\max}=0.999$ using an adaptive fifth-order Dormand--Prince solver. Both relative and absolute tolerances are set to $10^{-5}$, and the exponential-moving-average parameters are used for generation. Each generated table contains $n_{\mathrm{train}}$ observations.
\begin{figure*}[t]
\centering
\includegraphics[width=\textwidth]
{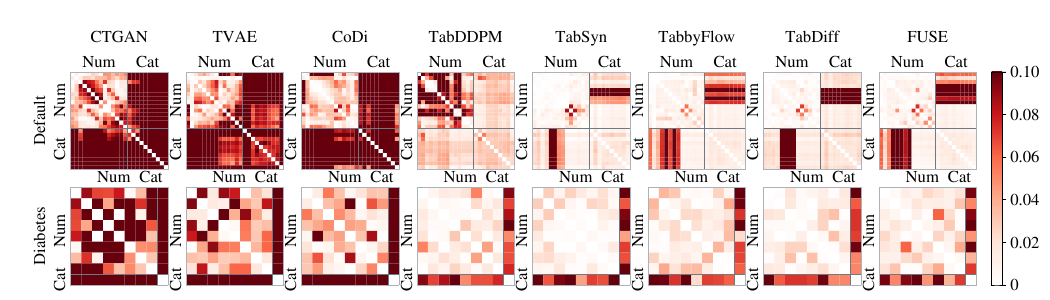}
\caption{Absolute differences between pairwise correlations from real and synthetic data (Default and Diabetes). Values closer to zero indicate more faithful preservation of feature correlation.}
\label{fig:supp-additional-heatmaps}
\end{figure*}

\section{Additional Experiments}
\label{app:additional-experiments}

The following experiments extend the main evaluation with support-based metrics, and additional visual comparisons.

\subsection{$\alpha$-Precision and $\beta$-Recall}
Table~\ref{tab:supp-alpha-beta} reports $\alpha$-Precision and $\beta$-Recall across all eight datasets. The former evaluates whether synthetic observations remain within high-probability regions of the real distribution, while the latter evaluates how well the synthetic distribution covers the support of the real data \citep{alaa2022faithful}. Reporting both metrics distinguishes fidelity to the data support from support coverage.

FUSE achieves the best average rank for $\alpha$-Precision at $1.50$ and the second-best average rank for $\beta$-Recall at $2.75$, following TabbyFlow at $2.13$. It also obtains the highest $\beta$-Recall on Diabetes and Fault. These results show that the strong support precision of FUSE is accompanied by competitive coverage of the real distribution.

\subsection{Additional Synthetic Data Visualizations}

Figures~\ref{fig:supp-numerical-marginals} and~\ref{fig:supp-categorical-marginals} compare the marginal distributions generated by all eight methods on Default and Diabetes, which are not included in the main marginal visualizations. Each numerical plotting range is determined from the corresponding real training data, while every category observed in the real data is retained in the categorical panels.

FUSE closely follows the numerical distribution in Default and the unimodal distribution in Diabetes. In both cases, it reproduces the dominant probability mass while retaining the overall tail behavior. These observations extend the marginal comparisons in the main paper to two additional datasets and are consistent with the aggregate Shape results.

For categorical variables, FUSE preserves both the dominant and lower-frequency levels in Default and closely reproduces the binary class balance in Diabetes. The categorical panels complement the numerical comparisons by illustrating the two components summarized by the Shape metric.

Figure~\ref{fig:supp-additional-heatmaps} extends the dependence analysis to Default and Diabetes and includes the four baselines omitted from the main heatmap. Unlike Trend, which averages similarities over pairs selected according to their dependence in the real data, the heatmaps display every variable pair. Numerical pairs report one half of the absolute difference between the real and synthetic Pearson correlations. Pairs involving a categorical variable report the total variation distance between the corresponding contingency tables. Numerical variables in mixed-type pairs are discretized into ten common bins before the contingency tables are constructed. The shared color scale is capped at $0.1$, matching the main heatmap.
\subsection{Extended Component Analysis}
\label{app:complete-ablation}

\begin{table*}[t]
\centering
\small
\setlength{\tabcolsep}{4.6pt}
\renewcommand{\arraystretch}{0.92}
\begin{tabular}{rrrrrrr}
\toprule
$\rho_{\mathrm{dep}}$ & Shape & Trend & C2ST & MLE &
$\alpha$-Precision & $\beta$-Recall \\
\midrule
0.00 & 0.000(.001) & 0.002(.008) & 0.002(.005) & 0.003(.004) & 0.001(.009) & 0.000(.004) \\
0.20 & 0.000(.001) & 0.002(.007) & 0.000(.007) & 0.104(.004) & -0.003(.007) & 0.000(.004) \\
0.40 & 0.001(.002) & 0.008(.012) & 0.002(.004) & 0.239(.020) & 0.015(.020) & 0.012(.006) \\
0.60 & 0.003(.004) & 0.015(.019) & 0.002(.004) & 0.364(.015) & 0.016(.016) & 0.035(.011) \\
0.80 & 0.000(.001) & 0.028(.018) & 0.002(.004) & 0.482(.035) & 0.004(.015) & 0.102(.005) \\
0.95 & 0.003(.002) & 0.010(.020) & 0.000(.006) & 0.508(.015) & 0.001(.004) & 0.161(.002) \\
\bottomrule
\end{tabular}

\caption{Effects of joint attention across dependence levels. Entries are
means with sample standard deviations over five matched seeds. Positive
values favor joint attention.}
\label{tab:supp-dependence-effects}
\end{table*}

The component analysis in the main paper separates adaptive mixture processing from joint attention across datasets. This section examines joint attention more directly through a controlled cross-type dependence experiment and matched comparisons on the eight datasets.

\begin{figure}[t]
\centering
\includegraphics[width=0.92\columnwidth]
{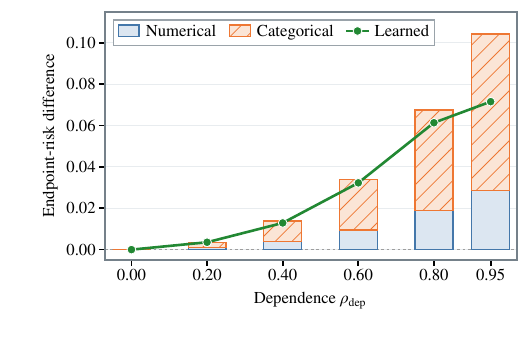}
\caption{Population conditioning penalty and learned endpoint-risk difference across dependence levels. Stacked bars show the numerical and categorical components of the population penalty. Green markers show the mean learned difference. Positive values favor joint attention.}
\label{fig:conditioning-risk}
\end{figure}

\subsubsection{Controlled Cross-Type Dependence.}
\label{app:between-type-association}

To isolate the effect of cross-type dependence, we construct a mixed-type
distribution in which $\rho_{\mathrm{dep}}$ controls the association between
a binary variable and a numerical variable. For
$\rho_{\mathrm{dep}}\in\{0,0.2,0.4,0.6,0.8,0.95\}$, let
\begin{align}
C^\star
&\sim \operatorname{Bernoulli}(1/2),
\qquad
\varepsilon\sim\mathcal N(0,1),
\qquad
\varepsilon\perp C^\star,
\notag\\
Z^\star
&=
\rho_{\mathrm{dep}}(2C^\star-1)
+
\sqrt{1-\rho_{\mathrm{dep}}^2}\,\varepsilon.
\label{eq:synthetic-dependence}
\end{align}
This construction gives
$\operatorname{Corr}(Z^\star,2C^\star-1)=\rho_{\mathrm{dep}}$.
Each observation also contains two independent standard Gaussian variables
and one independent categorical variable. The binary variable
$C^\star$ serves as the downstream target, making its association with
$Z^\star$ the only cross-type dependence.

\begin{table*}[t]
\centering
\setlength{\tabcolsep}{3.8pt}
\renewcommand{\arraystretch}{1.0}
\begin{tabular}{lrrrrrr}
\toprule
Dataset & Shape & Trend & C2ST & MLE & $\alpha$-Precision & $\beta$-Recall \\
\midrule
Adult & 0.001(.001) & 0.030(.001) & 0.001(.005) & 0.034(.001) & 0.004(.001) & 0.295(.004) \\
Default & 0.000(.000) & 0.025(.020) & 0.007(.011) & 0.019(.007) & 0.002(.002) & 0.139(.004) \\
Beijing & 0.000(.001) & 0.032(.003) & -0.001(.003) & 0.164(.004) & 0.007(.001) & 0.047(.003) \\
Shoppers & -0.001(.001) & 0.006(.002) & 0.004(.026) & 0.270(.033) & 0.007(.003) & 0.092(.005) \\
Magic & 0.000(.000) & 0.011(.005) & 0.007(.006) & 0.428(.015) & 0.058(.002) & 0.069(.007) \\
News & 0.000(.001) & 0.003(.000) & 0.000(.007) & 0.007(.001) & 0.001(.003) & 0.217(.007) \\
Diabetes & 0.002(.002) & 0.010(.003) & 0.003(.014) & 0.277(.044) & 0.029(.015) & 0.084(.025) \\
Fault & 0.000(.002) & 0.012(.001) & 0.007(.020) & 0.462(.034) & 0.022(.005) & 0.303(.021) \\
\bottomrule
\end{tabular}

\caption{Paired effects of joint attention on the six evaluation metrics.
Entries are means with sample standard deviations over five matched seeds
after averaging the adaptive and standard processing contrasts. Positive
values favor joint attention.}
\label{tab:supp-joint-attention-effects}
\end{table*}

Joint and restricted attention are compared over five matched seeds at each dependence level. Both configurations use adaptive mixture processing and differ only in their attention scope. Each seed contains $20{,}000$ training and $5{,}000$ test observations. The population conditioning penalty is estimated using $10^6$ Monte Carlo samples, while each pair of fitted models is evaluated on a separate sample of size $10^5$.

Figure~\ref{fig:conditioning-risk} shows that both the population penalty and the learned endpoint-risk reduction are near zero at independence and increase with cross-type dependence. At $\rho_{\mathrm{dep}}=0.95$, they reach $0.1043$ and $0.0715$, respectively. All five learned reductions are positive at every nonzero dependence level. Although the learned reductions compare fitted predictors rather than estimate the population penalty, their direction is consistent with Proposition~1.

Table~\ref{tab:supp-dependence-effects} shows that the MLE gain increases from $0.003$ to $0.508$ at $\rho_{\mathrm{dep}}=0.95$. In contrast, the effects remain at most $0.003$ for Shape and $0.002$ for C2ST. Joint attention therefore becomes most consequential for downstream utility as task-relevant cross-type dependence strengthens.

\subsubsection{Cross-Dataset Evidence.}
\label{app:cross-dataset-attention}

The controlled experiment is complemented by evaluations on the eight benchmark datasets. First, let $R_{j,d,s}^{\mathrm{cat}}$ denote the categorical endpoint risk on dataset $d$ under seed $s$ and attention configuration $j\in\{\mathrm{joint},\mathrm{restricted}\}$. The relative increase under restricted attention is defined as
\begin{align}
\delta_{d,s}^{\mathrm{cat}}
&=
100 \times
\frac{
R_{\mathrm{restricted},d,s}^{\mathrm{cat}}
-
R_{\mathrm{joint},d,s}^{\mathrm{cat}}
}{
R_{\mathrm{joint},d,s}^{\mathrm{cat}}
}.
\label{eq:relative-categorical-risk}
\end{align}

\begin{table}[t]
\centering
\small
\setlength{\tabcolsep}{4.0pt}
\renewcommand{\arraystretch}{0.90}
\begin{tabular}{lrr}
\toprule
Dataset & Relative increase (\%) & Positive pairs \\
\midrule
Adult & 4.650 (0.051) & 5/5 \\
Default & 15.722 (0.157) & 5/5 \\
Beijing & 5.009 (0.037) & 5/5 \\
Shoppers & 2.048 (0.179) & 5/5 \\
Magic & 9.630 (0.121) & 5/5 \\
News & 13.562 (0.806) & 5/5 \\
Diabetes & 10.826 (1.689) & 5/5 \\
Fault & 51.392 (1.916) & 5/5 \\
\bottomrule
\end{tabular}

\caption{Relative increase in categorical endpoint risk under restricted attention. Entries are means with sample standard deviations over five random seeds, together with the number of seeds yielding a positive paired difference.}
\label{tab:supp-categorical-endpoint-risk}
\end{table}

Both configurations use adaptive mixture processing and are evaluated using common interpolation draws over the same fixed held-out observations. Table~\ref{tab:supp-categorical-endpoint-risk} reports a positive mean increase on every dataset, with positive paired effects across all evaluations. The endpoint-risk advantage of joint attention therefore holds consistently across the datasets. 

The second analysis considers all six evaluation metrics under both processing configurations. Let $Y_{p,j,d,s}^{(m)}$ denote metric $m$ under processing configuration $p\in\{\mathrm{adaptive},\mathrm{standard}\}$ and attention configuration $j\in\{\mathrm{joint},\mathrm{restricted}\}$. To ensure that positive values consistently favor joint attention, define
\begin{align*}
\eta_d^{(m)}
&=
\begin{cases}
-1, & \text{if $m$ is regression MLE},\\
1,  & \text{otherwise},
\end{cases}
\label{eq:metric-orientation}
\end{align*}
and average the paired effect equally over the two processing configurations
as
\begin{align*}
\Delta_{d,s}^{(m)}
&=
\frac{\eta_d^{(m)}}{2}
\sum
\left(
Y_{p,\mathrm{joint},d,s}^{(m)}
-
Y_{p,\mathrm{restricted},d,s}^{(m)}
\right).
\end{align*}

Joint attention consistently improves Trend, $\beta$-Recall, and task-oriented MLE across the matched evaluations. As shown in Table~\ref{tab:supp-joint-attention-effects}, the mean effect on $\alpha$-Precision is positive for every dataset, whereas the effects on Shape and C2ST remain small and vary in sign. These results indicate that joint attention contributes most consistently to dependence preservation, support coverage, and downstream utility. Together with the controlled experiment, the cross-dataset results show that this benefit persists across heterogeneous benchmark datasets.

\end{document}